%% file: main.tex
\documentclass[10pt,twocolumn,letterpaper]{article}

\usepackage{iccv}
\usepackage{times}
\usepackage{epsfig}
\usepackage{graphicx}
\usepackage{amsmath}
\usepackage{amssymb}
\graphicspath{{./images/}}
\usepackage{makecell}
\usepackage{bm}
\usepackage{multirow} 
\usepackage{rotating}
\usepackage{subfigure}
\usepackage{amsthm}
\usepackage{thmtools}
\usepackage{thm-restate}
\usepackage{flexisym}
\usepackage{colortbl}
\usepackage{xcolor}
\usepackage[pagebackref=true,colorlinks,bookmarks=false]{hyperref}

\iccvfinalcopy 

\input{math_commands.tex}

\newtheorem{theorem}{Theorem}
\newtheorem{lemma}[theorem]{Lemma}

\newtheorem{definition}{Definition}

\newcommand{\BU}[1]{#1}

\newcommand{\BV}[1]{\color{black!10!blue}{\em #1}}
\newcommand{\redbold}[1]{{\textbf{#1}}}

\newcommand{\NameNoSpace}{M\textsuperscript{3}SDA}
\newcommand{\Name}{M\textsuperscript{3}SDA }
\newcommand{\NameBNoSpace}{M\textsuperscript{3}SDA-$\beta$}
\newcommand{\NameB}{M\textsuperscript{3}SDA-$\beta$ }
\newcommand{\DATASETNAME}{DomainNet }
\newcommand{\DATASETNAMENoSpace}{DomainNet}

\def\iccvPaperID{4125} 

\ificcvfinal\fi
\begin{document}

\title{Moment Matching for Multi-Source Domain Adaptation}

\author{Xingchao Peng\\
Boston University\\
{\tt\small xpeng@bu.edu}
\and
Qinxun Bai\\
Horizon Robotics\\
{\tt\small qinxun.bai@horizon.ai}
\and
Xide Xia\\
Boston University\\
{\tt\small xidexia@bu.edu}
\\
\and 
Zijun Huang\\
Columbia University\\
{\tt \small zijun.huang@columbia.edu}
\and 
Kate Saenko\\
Boston University\\
{\tt\small saenko@bu.edu}
\and
Bo Wang\\
{\normalsize Vector Institute \& Peter Munk Cardiac Center}\\
{\tt\small bowang@vectorinstitute.ai}
}

\maketitle

\begin{abstract}
\vspace{-0.3cm}
Conventional unsupervised domain adaptation (UDA) assumes that training data are sampled from a single domain. This neglects the more practical scenario where training data are collected from multiple sources, requiring multi-source domain adaptation. We make three major contributions towards addressing this problem. First, we collect and annotate by far the largest \textit{UDA} dataset, called DomainNet, which contains six domains and about 0.6 million images distributed among 345 categories, addressing the gap in data availability for multi-source UDA research. Second, we propose a new deep learning approach, Moment Matching for Multi-Source Domain Adaptation (\NameNoSpace), which aims to transfer knowledge learned from multiple labeled source domains to an unlabeled target domain by dynamically aligning moments of their feature distributions. Third, we provide new theoretical insights specifically for moment matching approaches in both single and multiple source domain adaptation. Extensive experiments are conducted to demonstrate the power of our new dataset in benchmarking state-of-the-art multi-source domain adaptation methods, as well as the advantage of our proposed model. Dataset and Code are available at \url{http://ai.bu.edu/M3SDA/}
\end{abstract}

\input{1_introduction}

\input{2_related}
\input{4_dataset}

\input{3_MSDA.tex}
\input{5_experiments.tex}

\input{6_conclusion.tex}

{\small
\bibliographystyle{ieee_fullname}
\bibliography{egbib}
}

\input{7_supp.tex}

\end{document}

%% file: math_commands.tex
\usepackage{amsmath,amsfonts,bm}
\usepackage{scalerel}

\def\eqref#1{equation~\ref{#1}}
\def\Eqref#1{Equation~\ref{#1}}

\def\1{\bm{1}}

\def\rvi{{\mathbf{i}}}

\def\rmP{{\mathbf{P}}}

\def\rmX{{\mathbf{X}}}

\def\valpha{{\bm{\alpha}}}
\def\vbeta{{\bm{\beta}}}

\def\vx{{\bm{x}}}

\DeclareMathAlphabet{\mathsfit}{\encodingdefault}{\sfdefault}{m}{sl}
\SetMathAlphabet{\mathsfit}{bold}{\encodingdefault}{\sfdefault}{bx}{n}

\def\gC{{\mathcal{C}}}
\def\gD{{\mathcal{D}}}

\def\gH{{\mathcal{H}}}

\def\gL{{\mathcal{L}}}

\def\gW{{\mathcal{W}}}
\def\gX{{\mathcal{X}}}

\newcommand{\E}{\mathbb{E}}

\newcommand{\R}{\mathbb{R}}
\newcommand{\N}{\mathbb{N}}

\DeclareMathOperator*{\argmin}{arg\,min}

%% file: 1_introduction.tex
\vspace{-0.3cm}
\section{Introduction}

Generalizing models learned on one visual domain to novel domains has been a major obstacle in the quest for universal object recognition. The performance of the learned models degrades significantly when testing on novel domains due to the presence of \textit{domain shift}~\cite{datashift_book2009}. 

\begin{figure}
    \centering
    \includegraphics[width=\linewidth]{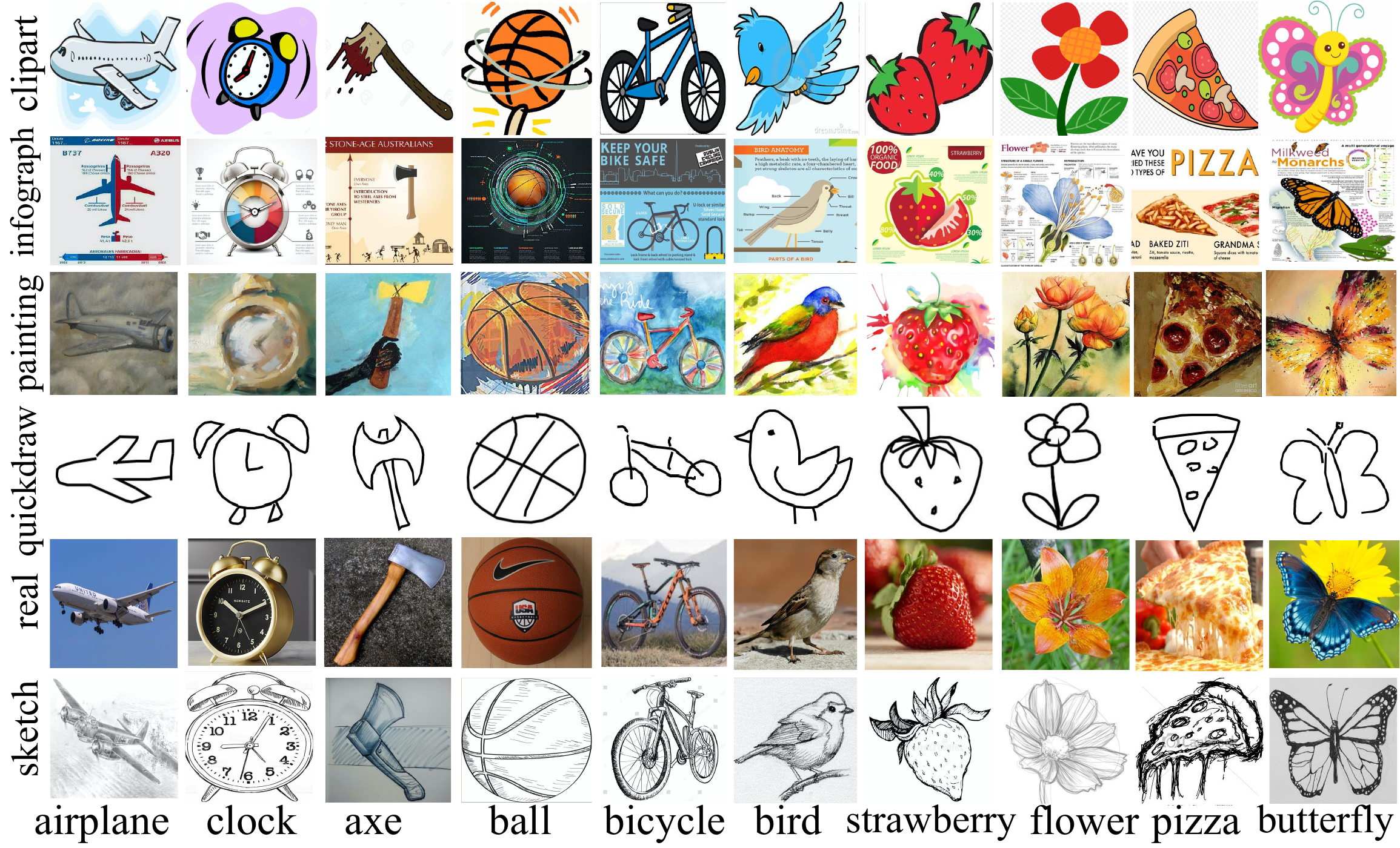}
    \caption{We address \textbf{Multi-Source Domain Adaptation} where source images come from multiple domains. We collect a large scale dataset, DomainNet, with six domains, 345 categories, and $\sim$0.6 million images and propose a model (\NameNoSpace) to transfer knowledge from multiple source domains to an unlabeled target domain.}

    \label{fig_msda_sample_image}
    \vspace{-0.3cm}
\end{figure}

Recently, transfer learning and domain adaptation methods have been proposed to mitigate the domain gap. For example, several UDA methods~\cite{JAN,tzeng2014deep,long2015} incorporate Maximum Mean Discrepancy loss into a neural network to diminish the domain discrepancy; other models introduce different learning schema to align the source and target domains, including aligning second order correlation~\cite{sun2015return,peng2017synthetic}, moment matching~\cite{zellinger2017central}, adversarial domain confusion~\cite{adda, DANN, MCD} and GAN-based alignment~\cite{CycleGAN2017,hoffman2017cycada,UNIT}. However, most of current UDA methods assume that source samples are collected from a single domain. This assumption neglects the more practical scenarios where labeled images are typically collected from multiple domains. For example, the training images can be taken under different weather or lighting conditions, share different visual cues, and even have different modalities (as shown in Figure~\ref{fig_msda_sample_image}). 

In this paper, we consider multi-source domain adaptation (MSDA), a more difficult but practical problem of knowledge transfer from multiple distinct domains to one unlabeled target domain. The main challenges in the research of MSDA are that: (1) the source data has multiple domains, which hampers the effectiveness of mainstream single UDA method; (2) source domains also possess domain shift with each other; (3) the lack of large-scale multi-domain dataset hinders the development of MSDA models.

In the context of MSDA, some theoretical analysis~\cite{ben2010theory,Mansour_nips2018,crammer2008learning,Zhao2017MultipleSD, NIPS2018_8046} has been proposed for multi-source domain adaptation (MSDA). Ben-David et al~\cite{ben2010theory} pioneer this direction by introducing an $\mathcal{H}\Delta\mathcal{H}$-divergence between the weighted combination of source domains and target domain. More applied works~\cite{duan2012exploiting, xu2018deep} use an adversarial discriminator to align the multi-source domains with the target domain. However, these works focus only on aligning the source domains with the target, neglecting the domain shift between the source domains. Moreover, $\mathcal{H}\Delta\mathcal{H}$-divergence based analysis does not directly correspond to moment matching approaches.

\begin{table}[t]
\small
\setlength{\tabcolsep}{0.1em}
    \centering
    \begin{tabular}{c|c| c c c| c}
    Dataset & Year &\footnotesize{Images} & \footnotesize{Classes} & \footnotesize{Domains} & \textit{Description} \\
    \hline
    
    \footnotesize{Digit-Five}& - &$\sim$100,000 & 10 & 5 & digit \\
    \footnotesize{Office~\cite{office}} & 2010 & 4,110 & 31 & 3 & office\\
    \footnotesize{Office-Caltech~\cite{gong2012geodesic}}& 2012 & 2,533 & 10 & 4 & office \\
     \footnotesize{CAD-Pascal~\cite{peng2015learning}}& 2015 & 12,000 & 20 & 6 & \footnotesize{animal,vehicle} \\
    \footnotesize{Office-Home~\cite{officehome}} & 2017 & 15,500 & 65 & 4 
    & office, home\\
    \footnotesize{PACS}~\cite{PACS} & 2017 & 9,991 & 7 & 4 & animal, stuff\\
    \footnotesize{Open MIC~\cite{openmic}} & 2018 &  16,156 & - & - & museum\\
    \footnotesize{Syn2Real~\cite{syn2real}} & 2018 &  280,157 & 12 & 3 & \footnotesize{ animal,vehicle}\\ 
    \hline
    \textbf{\DATASETNAMENoSpace} (Ours)&-& \textbf{569,010}& \textbf{345} & \textbf{6} & \footnotesize{see \textit{\url{Appendix}}}
    
    \end{tabular}
    \caption{A collection of most notable datasets to evaluate domain adaptation methods. Specifically, ``Digit-Five'' dataset indicates five most popular digit datasets (\textit{MNIST}~\cite{lecun1998gradient}, \textit{MNIST-M}~\cite{DANN}, Synthetic Digits~\cite{DANN}, \textit{SVHN}, and \textit{USPS}) which are widely used to evaluate domain adaptation models.  Our dataset is challenging as it contains more images, categories, and domains than other datasets. (see Table~\ref{dataset_1}, Table~\ref{dataset_2}, and Table~\ref{dataset_3} in \textit{\url{Appendix}} for detailed categories.)}
    \label{tab_dataset}
    \vspace{-0.4cm}
    
\end{table}

In terms of data, research has been hampered due to the lack of large-scale domain adaptation datasets, as
 state-of-the-art datasets contain only a few images or have a limited number of classes. Many domain adaptation models exhibit saturation when evaluated on these datasets. For example, many methods achieve $\sim$90 accuracy on the popular Office~\cite{office} dataset; Self-Ensembling~\cite{SE} reports $\sim$99\% accuracy on the ``Digit-Five'' dataset and $\sim$92\% accuracy on Syn2Real~\cite{syn2real} dataset.

In this paper, we first collect and label a new multi-domain dataset called \textbf{\DATASETNAMENoSpace}, aiming to overcome benchmark saturation. Our dataset consists of six distinct domains, 345 categories and $\sim$0.6 million images. A comparison of \DATASETNAME and several existing datasets is shown in Table~\ref{tab_dataset}, and example images are illustrated in Figure~\ref{fig_msda_sample_image}.  We evaluate several state-of-the-art single domain adaptation methods on our dataset, leading to surprising findings (see Section~\ref{sec_exp_all}). We also extensively evaluate our model on existing datasets and on  \DATASETNAME and show that it outperforms the existing single- and multi-source approaches.

Secondly, we propose a novel approach called \Name to tackle MSDA task by aligning the source domains with the target domain, and aligning the source domains with each other simultaneously. We dispose multiple complex adversarial training procedures presented in~\cite{xu2018deep}, but directly align the moments of their deep feature distributions, leading to a more robust and effective MSDA model. To our best knowledge, we are the first to empirically demonstrate that aligning the source domains is beneficial for MSDA tasks.

Finally, we extend existing theoretical analysis~\cite{ben2010theory,NIPS2018_8046,Zhao2017MultipleSD} to the case of moment-based divergence between source and target domains, which provides new theoretical insight specifically for moment matching approaches in 
domain adaptation, including our approach and many others.


 
\begin{figure*}[t]
    \centering
        \includegraphics[width=\linewidth]{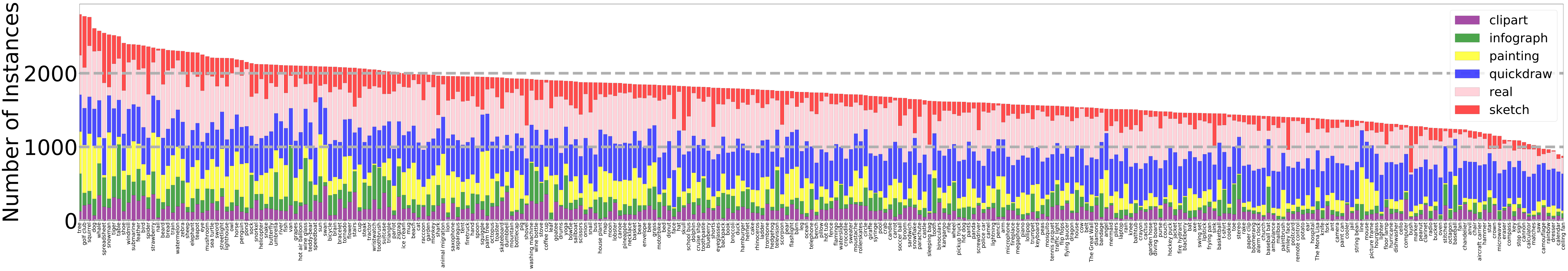}
        \includegraphics[width=\linewidth]{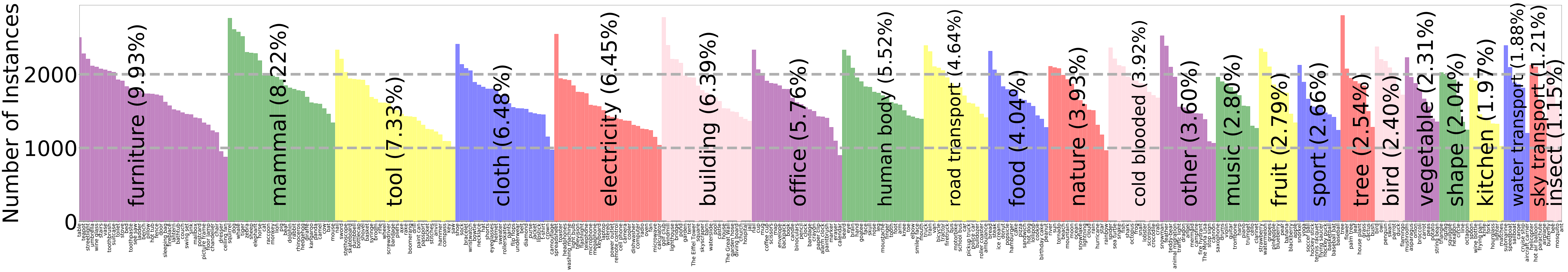}
    \caption{\textbf{Statistics for our \DATASETNAME dataset.} The two plots show object classes sorted by the total number of instances. The top figure shows the percentages each domain takes in the dataset. The bottom figure shows the number of instances grouped by 24 different divisions. Detailed numbers are shown in Table~\ref{dataset_1}, Table~\ref{dataset_2} and Table~\ref{dataset_3} in \textit{\url{Appendix}}. (Zoom in to see the exact class names!)} 
    \label{fig_number1}
    \centering
    \vspace{-0.3cm}
\end{figure*}

%% file: 2_related.tex
\section{Related Work}
\vspace{0.2em}

 \noindent \textbf{Domain Adaptation Datasets} Several notable datasets that can be utilized to evaluate domain adaptation approaches are summarized in Table~\ref{tab_dataset}. The Office dataset~\cite{office} is a popular benchmark for office environment objects. It contains 31 categories captured in three domains: office environment images taken with a high quality camera (DSLR), office environment images taken with a low quality camera (Webcam), and images from an online merchandising website (Amazon). The office dataset and its extension,  Office-Caltech10~\cite{gong2012geodesic}, have been used in numerous domain adaptation papers~\cite{long2015, adda, JAN, sun2015return, xu2018deep}, and the adaptation performance has reached $\sim$90\% accuracy. More recent benchmarks~\cite{officehome, openmic, visda} are proposed to evaluate the effectiveness of domain adaptation models. However, these datasets are small-scale and limited by their specific environments, such as \textit{office}, \textit{home}, and \textit{museum}.  Our dataset contains about 600k images, distributed in 345 categories and 6 distinct domains. We capture various object divisions, ranging from \textit{furniture}, \textit{cloth}, \textit{electronic} to \textit{mammal, building}, etc. 
 
 \vspace{0.2em}
 \noindent \textbf{Single-source UDA}  Over the past decades, various single-source UDA methods have been proposed. These methods can be taxonomically divided into three categories. The first category is the discrepancy-based DA approach, which utilizes different metric learning schemas to diminish the domain shift between source and target domains. Inspired by the kernel two-sample test~\cite{gretton2007kernel}, Maximum Mean Discrepancy (MMD) is applied to reduce distribution shift in various methods~\cite{JAN,tzeng2014deep,ghifary2014domain, meda}. Other commonly used methods include correlation alignment~\cite{sun2015return,peng2017synthetic}, Kullback-Leibler (KL) divergence~\cite{zhuang2015supervised}, and $\mathcal{H}$ divergence~\cite{ben2010theory}. The second category is the adversarial-based approach~\cite{cogan, adda}. A domain discriminator is leveraged to encourage the domain confusion by an adversarial objective. Among these approaches, generative adversarial networks are widely used to learn domain-invariant features as well to generate fake source or target data. Other frameworks utilize only adversarial loss to bridge two domains. The third category is reconstruction-based, which assumes the data reconstruction helps the DA models to learn domain-invariant features. The reconstruction is obtained via an encoder-decoder~\cite{bousmalis2016domain,ghifary2016deep} or a GAN discriminator, such as dual-GAN~\cite{yi2017dualgan}, cycle-GAN~\cite{CycleGAN2017}, disco-GAN~\cite{kim2017learning}, and CyCADA~\cite{hoffman2017cycada}. Though these methods make progress on UDA, few of them consider the practical scenario where training data are collected from multiple sources. Our paper proposes a model to tackle multi-source domain adaptation, which is a more general and challenging scenario.
 
  \vspace{0.2em}
 \noindent \textbf{Multi-Source Domain Adaptation}  Compared with single source UDA, multi-source domain adaptation assumes that training data from multiple sources are available. Originated from the early theoretical analysis~\cite{ben2010theory,Mansour_nips2018,crammer2008learning}, MSDA has many practical applications~\cite{xu2018deep,duan2012exploiting}. Ben-David et al~\cite{ben2010theory} introduce an $\mathcal{H}\Delta\mathcal{H}$-divergence between the weighted combination of source domains and target domain. Crammer et al~\cite{crammer2008learning} establish a general bound on the expected loss of the model by minimizing the empirical loss on the nearest \textit{k} sources. Mansour et al~\cite{Mansour_nips2018} claim that the target hypothesis can be represented by a weighted combination of source hypotheses. In the more applied works, Deep Cocktail Network (DCTN)~\cite{xu2018deep} proposes a \textit{k}-way domain discriminator and category classifier for digit classification and real-world object recognition. Hoffman et al~\cite{NIPS2018_8046} propose normalized solutions with theoretical guarantees for cross-entropy loss, aiming to provide a solution for the MSDA problem with very practical benefits. Duan et al~\cite{duan2012exploiting} propose \textit{Domain Selection Machine} for event recognition in consumer videos by leveraging a large number of loosely labeled web images from different sources. Different from these methods, our model directly matches all the distributions by matching the moments. Moreover, we provide a concrete proof of why matching the moments of multiple distributions works for multi-source domain adaptation.

 \vspace{0.2em}
 \noindent \textbf{Moment Matching}  The moments of distributions have been studied by the machine learning community for a long time. In order to diminish the domain discrepancy between two domains, different moment matching schemes have been proposed. For example, MMD matches the first moments of two distributions. Sun et al~\cite{sun2015return} propose an approach that matches the second moments. Zhang et al~\cite{zhang2018aligning} propose to align infinte-dimensional covariance matrices in RKHS. Zellinger et al~\cite{zellinger2017central} introduce a moment matching regularizer to match high moments. As the generative adversarial network (GAN) becomes popular, many GAN-based moment matching approaches have been proposed. \textit{McGAN} ~\cite{mroueh2017mcgan} utilizes a GAN to match the mean and covariance of feature distributions. GMMN~\cite{li2015generative} and MMD GAN~\cite{li2017mmd} are proposed for aligning distribution moments with generative neural networks. Compared to these methods, our work focuses on matching distribution moments for multiple domains and more importantly, we demonstrate that
 this is crucial for multi-source domain adaptation.  

%% file: 4_dataset.tex
\section{The DomainNet dataset}
\label{sec_dataset}

\begin{figure*}[t]
    \centering
    \includegraphics[width=\linewidth]{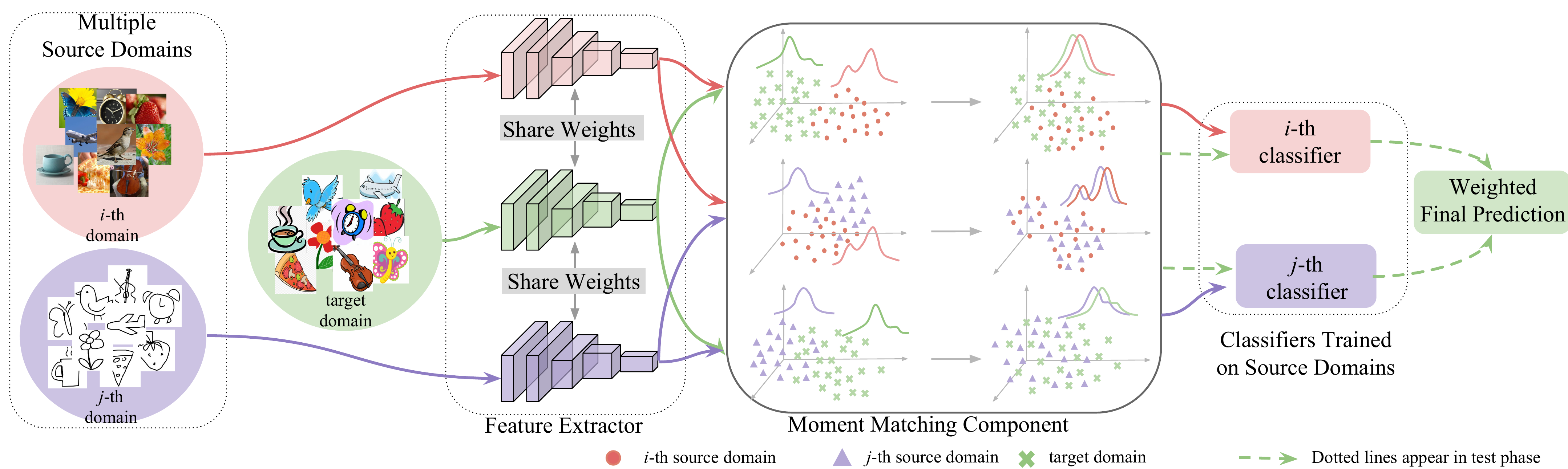}
    \vspace{-0.2cm}
    \caption{The framework of \textbf{Moment Matching for Multi-source Domain Adaptation} (\NameNoSpace). Our model consists of three components: i) feature extractor, ii) moment matching component, and iii) classifiers. Our model takes multi-source annotated training data as input and transfers the learned knowledge to classify the unlabeled target samples. Without loss of generality, we show the \textit{i}-th domain and \textit{j}-th domain as an example. The feature extractor maps the source domains into a common feature space. The moment matching component attempts to match the \textit{i}-th and \textit{j}-th domains with the target domain, as well as matching the \textit{i}-th domain with the \textit{j}-th domain. The final predictions of target samples are based on the weighted outputs of the \textit{i}-th and \textit{j}-th classifiers. (Best viewed in color!)}
    \label{fig_over_view}
    \vspace{-0.2cm}
\end{figure*}

It is well-known that deep models require massive amounts of training data. Unfortunately, existing datasets for visual domain adaptation are either small-scale or limited in the number of categories. We collect by far the largest domain adaptation dataset to date, \textbf{\DATASETNAME}. The \DATASETNAME contains six domains, with each domain containing 345 categories of common objects, as listed in Table~\ref{dataset_1}, Table~\ref{dataset_2}, and Table~\ref{dataset_3} (see \textit{\url{Appendix}}). The domains include \textbf{Clipart} ({\BV{clp}}, see \textit{\url{Appendix}}, Figure~\ref{fig_clipart}): collection of clipart images; \textbf{Infograph} ({\BV{inf}}, see Figure~\ref{fig_infograph}): infographic images with specific object; \textbf{Painting} ({\BV{pnt}}, see Figure~\ref{fig_painting}): artistic depictions of objects in the form of paintings; \textbf{Quickdraw} ({\BV{qdr}}, see Figure~\ref{fig_quickdraw}): drawings of the worldwide players of game ``Quick Draw!"\footnote{https://quickdraw.withgoogle.com/data}; \textbf{Real} ({\BV{rel}}, see Figure~\ref{fig_real}): photos and real world images; and \textbf{Sketch} ({\BV{skt}}, see Figure~\ref{fig_sketch}): sketches of specific objects.

The images from \textit{clipart}, \textit{infograph}, \textit{painting}, \textit{real}, and \textit{sketch} domains are collected by searching a category name combined with a domain name (\eg ``aeroplane painting'') in different image search engines. One of the main challenges is that the downloaded data contain a large portion of outliers. To clean the dataset, we hire 20 annotators to manually filter out the outliers. This process took around 2,500 hours (more than 2 weeks) in total. To control the annotation quality,  we assign two annotators to each image, and only take the images agreed by both annotators. After the filtering process, we keep 423.5k images from the 1.2 million images crawled from the web. The dataset has an average of around 150 images per category for \textit{clipart} and \textit{infograph} domain, around 220 per category for \textit{painting} and \textit{sketch} domain, and around 510 for \textit{real} domain. A statistical overview of the dataset is shown in Figure~\ref{fig_number1}.

The \textit{quickdraw} domain is downloaded directly from \url{https://quickdraw.withgoogle.com/}. The raw data are presented as a series of discrete points with temporal information. We use the B-spline~\cite{de1978practical} algorithm to connect all the points in each strike to get a complete drawing. We choose 500 images for each category to form the \textit{quickdraw} domain, which contains 172.5k images in total.

%% file: 3_MSDA.tex
\vspace{-0.2cm}

\section{Moment Matching for Multi-Source  DA}
\vspace{-0.1cm}
\label{sec:theory}

\label{sec:method}

Given $\gD_{S} = \{\gD_1, \gD_2, ... , \gD_N\}$ the collection of labeled source domains and $\gD_T$ the unlabeled target domain, where all domains are defined by bounded rational measures on input space $\gX$, the multi-source domain adaptation problem aims to find a hypothesis in the given hypothesis space $\gH$, which minimizes the testing target error on $\gD_T$.

\vspace{-0.1cm}
\begin{definition}
\label{def_md}
Assume $\rmX_1$, $\rmX_2$ $, ...,$ $\rmX_N$, $\rmX_T$ are collections of i.i.d. samples from $\gD_1, \gD_2, ..., \gD_N, \gD_T$ respectively, then the Moment Distance between $\gD_{S}$ and $\gD_T$ is defined as
\begin{eqnarray}
\label{eqn:MD2}
    MD^2(\gD_S, \gD_T) = \sum_{k=1}^{2}\Bigg( \frac{1}{N}\sum_{i=1}^{N} \|\E(\rmX_i^k) - \E(\rmX_T^k) \|_2 \nonumber\\ 
       + \binom{N}{2}^{-1}\sum_{i=1}^{N-1} \sum_{j=i+1}^{N}\| \E(\rmX_i^k) - \E(\rmX_j^k)  \|_2 \Bigg).
\end{eqnarray}
\end{definition}

\vspace{-0.2cm}
\noindent \textbf{\Name}We propose a moment-matching model for MSDA based on deep neural networks. As shown in Figure~\ref{fig_over_view}, our model comprises of a feature extractor $G$, a moment-matching component, and a set of $N$ classifiers $\gC = \{C_1, C_2, ..., C_N\}$. The feature extractor $G$ maps $\gD_S$, $\gD_T$ to a common latent feature space. The moment matching component  minimizes the moment-related distance defined in Equation~\ref{eqn:MD2}. The $N$ classifiers are trained on the annotated source domains with cross-entropy loss. The overall objective function is:
\begin{equation}
\label{equ_objective}
    \min_{G,\gC}\sum_{i=1}^{N}\gL_{\gD_i} + \lambda \min_{G}MD^2(\gD_S,\gD_T),
\end{equation}
where $\gL_{\gD_i}$ is the softmax cross entropy loss for the classifier $C_i$ on domain $\gD_i$, and $\lambda$ is the trade-off parameter. 

\Name assumes that $p(y|x)$ will be aligned automatically when aligning $p(x)$, which might not hold in practice. To mitigate this limitation, we further propose \NameBNoSpace.

\vspace{0.4em}
\noindent \textbf{\NameB}In order to align $p(y|x)$ and $p(x)$ at the same time, we follow the training paradigm proposed by~\cite{MCD}. In particular, we leverage two classifiers per domain to form $N$ pairs of classifiers $\gC\textprime = \big\{(C_1,C_1\textprime),(C_2,C_2\textprime), ..., (C_N,C_N\textprime)\big\}$. The training procedure includes three steps. \textbf{i).} We train $G$ and $\gC\textprime$ to classify the multi-source samples correctly. The objective is similar to Equation~\ref{equ_objective}. \textbf{ii).} We then train the classifier pairs for a fixed $G$. The goal is to make the discrepancy of each pair of classifiers as large as possible on the target domain. For example, the outputs of $C_1$ and $C_1\textprime$ should possess a large discrepancy. Following~\cite{MCD}, we define the discrepancy of two classifiers as the L1-distance between the outputs of the two classifiers. The objective is:
\begin{equation}
\label{equ_objective_step2}
     \min_{\gC\textprime}\sum_{i=1}^{N}\gL_{\gD_i} -  \sum_{i}^N |P_{C_i}(D_T)-P_{C_i\textprime}(D_T)|,
\end{equation}
\noindent where $P_{C_i}(D_T)$, $P_{C_i\textprime}(D_T)$ denote the outputs of $C_i$, $C_i\textprime$ respectively on the target domain.  \textbf{iii).} Finally, we fix $\gC\textprime$ and train $G$ to minimize the discrepancy of each classifier pair on the target domain. The objective function is as follows:
\begin{equation}
\label{equ_objective_step3}
     \min_{G} \sum_{i}^N |P_{C_i}(D_T)-P_{C_i\textprime}(D_T)|,
\end{equation}
These three training steps are performed periodically until the whole network converges. 

\noindent \textbf{Ensemble Schema} In the testing phase, testing data from the target domain are forwarded through the feature generator and the $N$ classifiers. We propose two schemas to combine the outputs of the classifiers: 
\begin{itemize}
    \item average the outputs of the classifiers, marked as \NameNoSpace$^{\ast}$
    \item Derive a weight vector $\gW = (w_1,\ldots,w_{N-1})$ ($\sum_{i=1}^{N-1} w_i=1$, assuming $N$-th domain is the target). The final prediction is the weighted average of the outputs.
\end{itemize}
 To this end, how to derive the weight vector becomes a critical problem. The main philosophy of the weight vector is to make it represent the intrinsic closeness between the target domain and source domains. In our setting, the weighted vector is derived by the source-only accuracy between the $i$-th domain and the $N$-th domain, \textit{i.e.} $w_i = acc_i/\sum_{j=1}^{N-1}acc_j$.

\subsection{Theoretical Insight}
Following~\cite{ben2010theory}, we introduce a rigorous model of multi-source domain adaptation for binary classification. A domain $\gD=(\mu, f)$ is defined by a probability measure (distribution) $\mu$ on the input space $\gX$ and a labeling function $f:\gX\to\{0,1\}$.
A hypothesis is a function $h:\gX\to\{0,1\}$.
The probability that $h$ disagrees with the domain labeling function $f$ under the domain distribution $\mu$ is defined as:
\begin{equation}
\label{eq_def_error}
    \epsilon_{\gD}(h)=\epsilon_{\gD}(h,f)=\E_{\mu}\big[|h(\vx)-f(\vx)|\big].
\end{equation}

For a source domain $\gD_S$ and a target domain $\gD_T$, we refer to the source error and the target error of a hypothesis $h$ as $\epsilon_S(h)=\epsilon_{\gD_S}(h)$ and $\epsilon_T(h)=\epsilon_{\gD_T}(h)$ respectively.
When the expectation in Equation~\ref{eq_def_error} is computed with respect to an empirical distribution, we denote the corresponding empirical error by $\hat{\epsilon}_{\gD}(h)$, such as $\hat{\epsilon}_S(h)$ and $\hat{\epsilon}_T(h)$. In particular, we examine algorithms that minimize convex combinations of source errors, i.e., given a weight vector $\valpha=(\alpha_1,\ldots,\alpha_N)$ with $\sum_{j=1}^N\alpha_j=1$, we define the $\valpha$-weighted source error of a hypothesis $h$ as $\epsilon_{\valpha}(h)=\sum_{j=1}^N\alpha_j\epsilon_j(h),$
where $\epsilon_j(h)$ is the shorthand of $\epsilon_{\gD_j}(h)$. The empirical $\valpha$-weighted source error can be defined analogously and denoted by $\hat{\epsilon}_{\valpha}(h)$. 

Previous theoretical bounds~\cite{ben2010theory,NIPS2018_8046,Zhao2017MultipleSD} on the target error are based on the $\mathcal{H}\Delta\mathcal{H}$-divergence between the source and target domains. While providing theoretical insights for general multi-source domain adaptation, these $\mathcal{H}\Delta\mathcal{H}$-divergence based bounds do not directly motivate moment-based approaches. In order to provide a specific insight for moment-based approaches, we introduce the $k$-th order cross-moment divergence between domains, denoted by $d_{CM^{k}}(\cdot,\cdot)$, and extend the analysis in~\cite{ben2010theory} to derive the following moment-based bound for multi-source domain adaptation. See \textit{\url{Appendix}} for the definition of the cross-moment divergence and the proof of the theorem.



\begin{table*}[t]
\centering
\footnotesize
{
\begin{tabular}{l|c| c| c |c |c |c|c}
\multirow{2}{1cm}{Standards} &
\multirow{2}{1cm}{Models} & 
\multirow{2}{1.4cm}{{\BV{mt,up,sv,sy $\rightarrow$ mm}} } &  
\multirow{2}{1.4cm}{{\BV{mm,up,sv,sy $\rightarrow$ mt}} } & 
\multirow{2}{1.5cm}{{\BV{mm,mt,sv,sy $\rightarrow$ up}} }& 
\multirow{2}{1.6cm}{{\BV{mm,mt,up,sy $\rightarrow$ sv}} }&    
\multirow{2}{1.6cm}{{\BV{mm,mt,up,sv $\rightarrow$ sy}} } & 
\multirow{2}{1cm}{{\BU{Avg}}} \\ 
&&&&&&& \\
 \Xhline{0.7pt} 
\multirow{3}{*}{ \begin{tabular}[c]{@{}c@{}}Source\\Combine\end{tabular} } 
                                    &Source Only & 63.70$\pm$0.83 & 92.30$\pm$0.91 & 90.71$\pm$0.54 & 71.51$\pm$0.75 & 83.44$\pm$0.79 &\BU{80.33}$\pm$0.76\\
                                     &DAN~\cite{long2015} & 67.87$\pm$0.75 & 97.50$\pm$ 0.62  & 93.49$\pm$0.85 & 67.80$\pm$0.84 & 86.93$\pm$0.93 & \BU{82.72}$\pm$ 0.79 \\
                                      &DANN~\cite{DANN} & 70.81$\pm$0.94 & 97.90$\pm$0.83 & 93.47$\pm$0.79 & 68.50$\pm$0.85 & 87.37$\pm$0.68 & \BU{83.61}$\pm$0.82\\
                                     
\Xhline{0.7pt}
\multirow{9}{*}{ \begin{tabular}[c]{@{}c@{}}Multi-\\Source\end{tabular} } 
                                      &Source Only & 63.37$\pm$0.74 & 90.50$\pm$0.83 & 88.71$\pm$0.89 & 63.54$\pm$0.93 & 82.44$\pm$0.65 &\BU{77.71}$\pm$0.81\\
                                      &DAN~\cite{long2015} & 63.78$\pm$0.71 & 96.31$\pm$0.54 & 94.24$\pm$0.87 & 62.45$\pm$0.72 & 85.43$\pm$0.77 & \BU{80.44}$\pm$0.72\\
                                      &CORAL~\cite{sun2015return} & 62.53$\pm$0.69 & 97.21$\pm$0.83 & 93.45$\pm$0.82 & 64.40$\pm$0.72 & 82.77$\pm$0.69 & \BU{80.07}$\pm$0.75 \\
                                      &DANN~\cite{DANN} & 71.30$\pm$0.56 & 97.60$\pm$0.75 & 92.33$\pm$0.85 & 63.48$\pm$0.79 & 85.34$\pm$0.84 & \BU{82.01}$\pm$0.76\\
                                      &JAN~\cite{JAN} & 65.88$\pm$0.68 & 97.21$\pm$0.73 & 95.42$\pm$0.77 & 75.27$\pm$0.71 & 86.55$\pm$0.64 &84.07$\pm$0.71  \\ 
                                      &ADDA~\cite{adda} & 71.57$\pm$ 0.52 & 97.89$\pm$0.84 & 92.83$\pm$0.74 & 75.48$\pm$0.48 & 86.45$\pm$0.62 & \BU{84.84}$\pm$0.64\\
                                      
                                      & DCTN~\cite{xu2018deep} & 70.53$\pm$1.24 & 96.23$\pm$0.82 & 92.81$\pm$0.27 & 77.61$\pm$0.41 & 86.77$\pm$0.78 & \BU{84.79}$\pm$0.72\\
                                      &MEDA~\cite{meda} & 71.31$\pm$0.75 & 96.47$\pm$0.78 & 97.01$\pm$0.82 & 78.45$\pm$0.77 & 84.62$\pm$0.79 & 85.60$\pm$ 0.78\\  
                                      &MCD~\cite{MCD} & 72.50$\pm$0.67 & 96.21$\pm$0.81 & 95.33$\pm$0.74 & 78.89$\pm$0.78 & 87.47$\pm$0.65 & 86.10$\pm$0.73\\
                                      &\textbf{ \Name} (ours)& 69.76$\pm$0.86 & \redbold{98.58}$\pm$0.47 & 95.23$\pm$0.79 & 78.56$\pm$0.95 & 87.56$\pm$0.53 & \BU{86.13}$\pm$0.64 \\
                                      & \textbf{\NameB} (ours)& \redbold{72.82}$\pm$1.13 & 98.43$\pm$0.68 &\redbold{ 96.14}$\pm$0.81 & \redbold{81.32}$\pm$0.86 & \redbold{89.58}$\pm$0.56 & \redbold{87.65}$\pm$ 0.75 \\
\end{tabular}
} 
\vspace{0.1cm}
\caption{
\textbf{Digits Classification Results.} \textbf{mt}, \textbf{up}, 
\textbf{sv}, \textbf{sy}, \textbf{mm} are abbreviations for \textit{MNIST}, \textit{USPS}, \textit{SVHN}, 
\textit{Synthetic Digits}, \textit{MNIST-M}, respectively. Our model 
\Name and \NameB achieve \textbf{86.13\%} and \textbf{87.65\%} accuracy, outperforming other baselines by a large margin.
} 
\label{tab_digit_five}
\end{table*}

\begin{restatable}{theorem}{pairwisebound}
\label{thm:pairwise_bound}
Let $\gH$ be a hypothesis space of $VC$ dimension $d.$ Let $m$ be the size of labeled samples from all sources $\{\gD_1, \gD_2, ... , \gD_N\}$, $S_j$ be the labeled sample set of size $\beta_jm$ ($\sum_j\beta_j=1$) drawn from $\mu_j$ and labeled by the groundtruth labeling function $f_j.$ If $\hat{h}\in\gH$ is the empirical minimizer of $\hat{\epsilon}_{\valpha}(h)$ for a fixed weight vector $\valpha$ and $h^*_T=\min_{h\in\gH}\epsilon_T(h)$ is the target error minimizer, then for any $\delta\in(0,1)$ and any $\epsilon>0$, there exist $N$ integers $\{n^j_\epsilon\}_{j=1}^N$ and $N$ constants $\{a_{n^j_\epsilon}\}_{j=1}^N$, such that with probability at least $1-\delta,$

\begin{equation}
\label{eqn:pair_bound}
\begin{split}
\epsilon_T(\hat{h})\leq& \ \epsilon_T(h^*_T)+\eta_{\valpha,\vbeta,m,\delta}
	+ \epsilon \\
	&+ \sum_{j=1}^N\alpha_j\bigg(2\lambda_j+a_{n^j_\epsilon}\sum_{k=1}^{n^j_\epsilon}d_{CM^{k}}\big(\gD_j, \gD_T\big)\bigg),
\end{split}
\end{equation}
where $\eta_{\valpha,\vbeta,m,\delta} = 4\sqrt{\Big(\sum_{j=1}^{N} \frac {\alpha_{j}^2} {\beta_{j}} \Big) \Big(\frac {2d(\log(\frac{2m}{d})+1)+2\log(\frac{4}{\delta})} {m}\Big) }$
and $\lambda_j=\min_{h\in\gH}\{\epsilon_T(h)+\epsilon_j(h)\}.$
\end{restatable}


Theorem~\ref{thm:pairwise_bound} shows that the upper bound on the target error of the learned hypothesis depends on the pairwise moment divergence $d_{CM^{k}}\big(\gD_S,\gD_T\big)$ between the target domain and each source domain.\footnote{Note that single source is just a special case when $N=1$.} 
This provides a direct motivation for moment matching approaches beyond ours. In particular, it motivates our multi-source domain adaptation approach to align the moments between each target-source pair. 
Moreover, it is obvious that the last term of the bound, $\sum_{k}d_{CM^{k}}\big(\gD_j, \gD_T\big)$, is lower bounded by the pairwise divergences between source domains. To see this, consider the toy example consisting of two sources $\gD_1,\gD_2$, and a target $\gD_T,$ since $d_{CM^{k}}(\cdot,\cdot)$ is a metric, triangle inequality implies the following lower bound:
$$d_{CM^{k}}\big(\gD_1, \gD_T\big) + d_{CM^{k}}\big(\gD_2, \gD_T\big)
\geq d_{CM^{k}}\big(\gD_1, \gD_2\big).$$
This motivates our algorithm to also align the moments between each pair of source domains. 
Intuitively, it is not possible to perfectly align the
target domain with every source domain, if the source
domains are not aligned themselves.
Further discussions of Theorem~\ref{thm:pairwise_bound} and its relationship with our algorithm are provided in the \textit{\url{Appendix}}.

%% file: 5_experiments.tex
\section{Experiments}
\label{sec_exp_all}


We perform an extensive evaluation on the following tasks: digit classification (\textit{MNIST, SVHN, USPS, MNIST-M, Sythetic Digits}), and image recognition (\textit{Office-Caltech10}, \DATASETNAME dataset). In total, we conduct 714 experiments. The experiments are run on a GPU-cluster with 24 GPUs and the total running time is more than 21,440 GPU-hours. Due to space limitations, we only report major results; more implementation details are provided in the supplementary material. Throughout the experiments, we set the trade-off parameter $\lambda$ in Equation ~\ref{equ_objective} as 0.5. In terms of the parameter sensitivity, we have observed that the performance variation is not significant if $\lambda$ is between 0.1$\sim$1. All of our experiments are implemented in the PyTorch\footnote{\url{http://pytorch.org}} platform.

\input{table_baseline.tex}
\begin{table}[t]
	\center
	\begin{footnotesize}
		\begin{tabular}{p{1.1cm}|c|p{0.45cm}p{0.45cm}p{0.45cm}p{0.45cm}|p{0.45cm}}
			\multirow{2}{0cm}{Standards} &\multirow{2}{0cm}{Models} &\multirow{2}{0cm}{\begin{tiny}A,C,D\\$\rightarrow$W\end{tiny}} &\multirow{2}{0.1cm}{\begin{tiny}A,C,W\\$\rightarrow$D\end{tiny}} &\multirow{2}{0cm}{\begin{tiny}A,D,W\\$\rightarrow$C\end{tiny}} &\multirow{2}{0cm}{\begin{tiny}C,D,W\\$\rightarrow$A\end{tiny}} &\multirow{2}{0cm}{Avg} \\ 
			&&&&&& \\
			\hline
			\multirow{2}{0.3cm}{Source Combine}	&Source only& 99.0	& 98.3& 87.8	& 86.1	  &	92.8	\\
			&DAN~\cite{long2015} 		& 99.3	& 98.2 & 89.7	& {
			\textbf{94.8}}	  &	95.5	\\
			\hline				
			\multirow{5}{0.3cm}{Multi-Source}
			&Source only & 99.1  & 98.2 & 85.4 &88.7 & 92.9 \\				
			&DAN~\cite{long2015}	&99.5	&	99.1	& 89.2	& 91.6 	& 94.8	\\
			&DCTN~\cite{xu2018deep} & 99.4 & 99.0 & 90.2 & 92.7 & 95.3\\
			&JAN~\cite{JAN}	&99.4	&	\textbf{99.4}	& 91.2	& 91.8 	& 95.5\\
			&MEDA~\cite{meda} & 99.3 & 99.2 & 91.4 & 92.9 & 95.7 \\
			&MCD~\cite{MCD} & 99.5 & 99.1 & 91.5 & 92.1 & 95.6 \\
			
			&\textbf{\Name}(ours)	&99.4		&99.2	& 91.5	&94.1 &	96.1	\\
			&\textbf{\NameB}(ours)	&	\textbf{99.5}	& 99.2	&  \textbf{92.2}	& 	94.5	 & 	\textbf{96.4}	\\
		\end{tabular}
	\end{footnotesize}
\caption{\textbf{Results on Office-Caltech10 dataset}. A,C,W and D represent \textit{Amazon}, \textit{Caltech}, \textit{Webcam} and \textit{DSLR}, respectively. All the experiments are based on ResNet-101 pre-trained on ImageNet.}
\label{tab_office_caltech}
 \vspace{-0.4cm}
\end{table}

\subsection{Experiments on Digit Recognition}
Five digit datasets are sampled from five different sources, namely \textit{MNIST}~\cite{lecun1998gradient}, \textit{Synthetic Digits}~\cite{DANN}, \textit{MNIST-M}~\cite{DANN}, \textit{SVHN}, and \textit{USPS}. Following \textit{DCTN}~\cite{xu2018deep}, we sample 25000 images from training subset and 9000 from testing subset in \textit{MNIST, MINST-M, SVHN}, and \textit{Synthetic Digits}. \textit{USPS} dataset contains only 9298 images in total, so we take the entire dataset as a domain. In all of our experiments, we take turns to set one domain as the target domain and the rest as the source domains.

We take four state-of-the-art discrepancy-based approaches: Deep Adaptation Network~\cite{long2015} (\textbf{DAN}), Joint Adaptation Network (\textbf{JAN}), Manifold Embedded Distribution Alignment (\textbf{MEDA}), and Correlation Alignment~\cite{sun2015return} (\textbf{CORAL}), and four adversarial-based approaches: Domain Adversarial Neural Network~\cite{DANN} (\textbf{DANN}), Adversarial Discriminative Domain Adaptation~\cite{adda} (\textbf{ADDA}), Maximum Classifier Discrepancy (\textbf{MCD}) and Deep Cocktail Network~\cite{xu2018deep} (\textbf{DCTN}) as our baselines. In the \textit{source combine} setting, all the source domains are combined to a single domain, and the baseline experiments are conducted in a traditional single domain adaptation manner.

The results are shown in Table~\ref{tab_digit_five}. Our model \Name achieves an \textbf{86.13\%} average accuracy, and \NameB boosts the performance to \textbf{87.65\%}, outperforming other baselines by a large margin. One interesting observation is that the results on MNIST-M dataset is lower. This phenomenon is probably due to the presence of \textit{negative transfer}~\cite{pan2010survey}. For a fair comparison, all the experiments are based on the same network architecture. For each experiment, we run the same setting for five times and report the mean and standard deviation. (See \textit{\url{Appendix}} for detailed experiment settings and analyses.)

\subsection{Experiments on Office-Caltech10}

The Office-Caltech10~\cite{gong2012geodesic} dataset is extended from the standard Office31~\cite{office} dataset. It consists of the same 10 object categories from 4 different domains: \textit{Amazon}, \textit{Caltech}, \textit{DSLR}, and \textit{Webcam}.

The experimental results on Office-Caltech10 dataset are shown in Table~\ref{tab_office_caltech}. Our model \Name gets a 96.1\% average accuracy on this dataset, and \NameB further boosts the performance to \textbf{96.4}\%. All the experiments are based on ResNet-101 pre-trained on ImageNet. As far as we know, our models achieve the best performance among all the results ever reported on this dataset. We have also tried AlexNet, but it did not work as well as ResNet-101.

\subsection{Experiments on \DATASETNAME}

\noindent \textbf{Single-Source Adaptation} To demonstrate the intrinsic difficulty of \DATASETNAMENoSpace, we evaluate multiple state-of-the-art algorithms for single-source adaptation: Deep Alignment Network (\textbf{DAN})~\cite{long2015}, Joint Adaptation Network (\textbf{JAN})~\cite{JAN}, Domain Adversarial Neural Network (\textbf{DANN})~\cite{DANN}, Residual Transfer Network (\textbf{RTN})~\cite{RTN}, Adversarial Deep Domain Adaptation (\textbf{ADDA})~\cite{adda}, Maximum Classifier Discrepancy (\textbf{MCD})~\cite{MCD}, and Self-Ensembling (\textbf{SE})~\cite{SE}. As the \DATASETNAME dataset contains 6 domains, experiments for 30 different (sources, target) combinations are performed for each baseline. For each domain, we follow a 70\%/30\% split scheme to participate our dataset into training and testing trunk. The detailed statistics can be viewed in Table~\ref{tab_split_detail} (see \textit{\url{Appendix}}). All other experimental settings (neural network, learning rate, stepsize, etc.) are kept the same as in the original papers. Specifically, DAN, JAN, DANN, and RTN are based on AlexNet~\cite{alexnet}, ADDA and MCD are based on ResNet-101~\cite{he2015deep}, and SE is based on ResNet-152~\cite{he2015deep}. Table~\ref{tab_challengeI} shows all the source-only and experimental results. (Source-only results for ResNet-101 and ResNet-152 are in \textit{\url{Appendix}}, Table~\ref{tab_resnet_source_only}). The results show that our dataset is challenging, especially for the \textit{infograph} and \textit{quickdraw} domain. We argue that the difficulty is mainly introduced by the large number of categories in our dataset.

 \begin{figure}[t]
    \centering
    \includegraphics[width=0.9\linewidth]{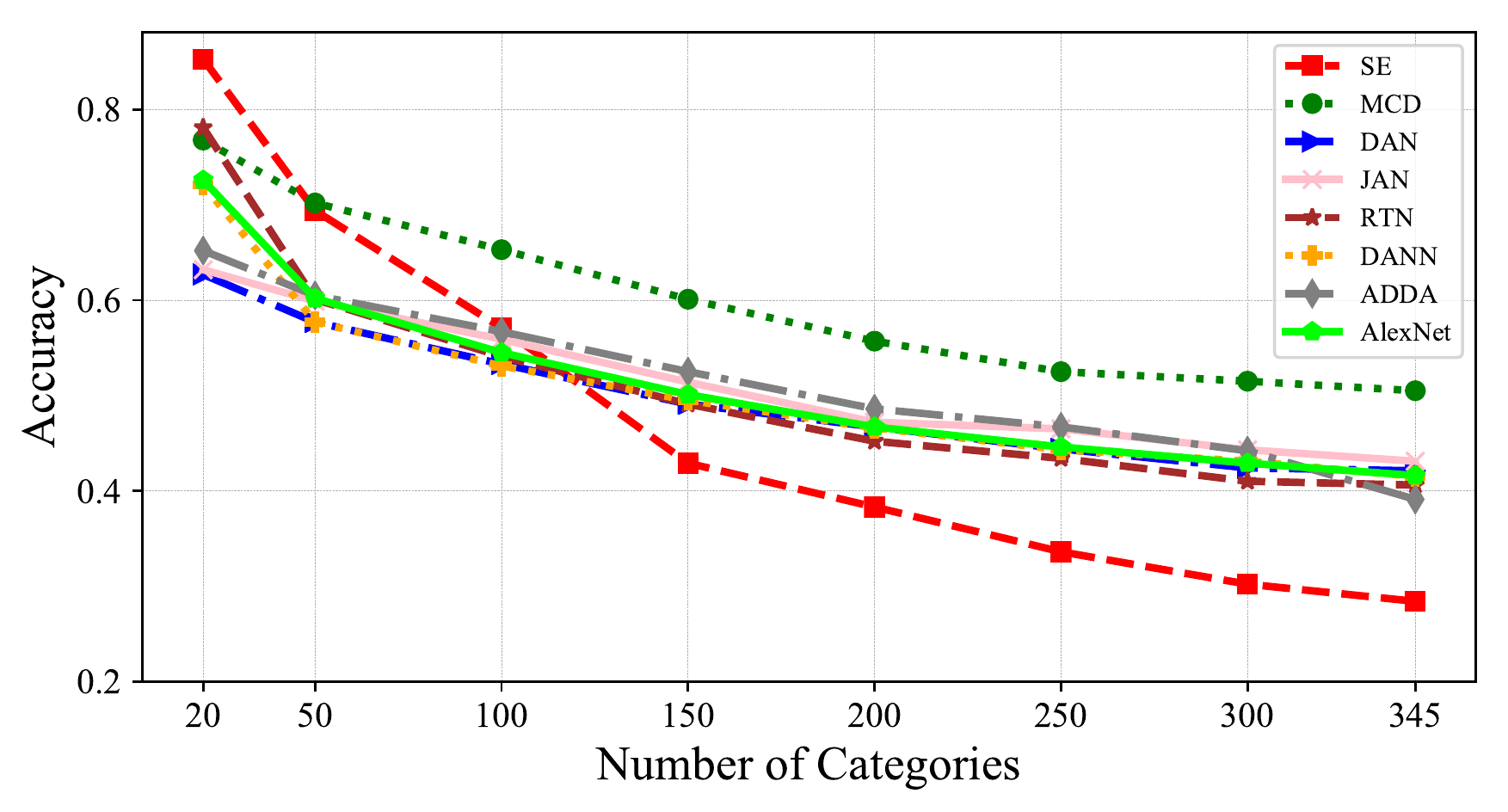}
    \caption{\textbf{Accuracy \textit{vs.} Number of categories.} This plot shows the \textit{painting}$\rightarrow$\textit{real} scenario. More plots with similar trend can be accessed in Figure ~\ref{fig_plot_acc_vs_num}  (see \textit{\url{Appendix}}).}
    \label{fig_relation_of_categories}
    \vspace{-0.45cm}
\end{figure}
 
 \vspace{.2em}
 \input{table_main_result.tex}
 
\noindent \textbf{Multi-Source Domain Adaptation} \DATASETNAME contains six domains. Inspired by Xu et al~\cite{xu2018deep}, we introduce two MSDA standards: (1) \textit{single best}, reporting the single best-performing source transfer
result on the test set, and (2) \textit{source combine}, combining the source domains to a single domain and performing traditional single-source adaptation. The first standard evaluates whether MSDA can improve the best single source UDA results; the second testify whether MSDA is necessary to exploit. 

\noindent \textbf{Baselines} For both \textit{single best} and \textit{source combine} experiment setting, we take the following state-of-the-art methods as our baselines: Deep Alignment Network (\textbf{DAN})~\cite{long2015}, Joint Adaptation Network (\textbf{JAN})~\cite{JAN}, Domain Adversarial Neural Network (\textbf{DANN})~\cite{DANN}, Residual Transfer Network (\textbf{RTN})~\cite{RTN}, Adversarial Deep Domain Adaptation (\textbf{ADDA})~\cite{adda}, Maximum Classifier Discrepancy (\textbf{MCD})~\cite{MCD}, and Self-Ensembling  (\textbf{SE})~\cite{SE}. For multi-source domain adaptation, we take Deep Cocktail Network (\textbf{DCTN})~\cite{xu2018deep} as our baseline.

\noindent \textbf{Results} The experimental results of multi-source domain adaptation are shown in Table~\ref{table_lsdac}. We report the results of the two different weighting schemas and all the baseline results in Table~\ref{table_lsdac}. Our model \Name achieves an average accuracy of \textbf{41.5\%}, and \NameB boosts the performance to \textbf{42.6\%}. The results demonstrate that our models designed for MSDA outperform the \textit{single best} UDA results, the \textit{source combine} results, and the multi-source baseline. From the experimental results, we make three interesting observations. (1)The performance of \NameNoSpace$^{\ast}$ is 40.8\%. After applying the weight vector $\gW$, ~\NameNoSpace improves the mean accuracy by 0.7 percent. (2) In {\BV{clp,inf,pnt,rel,skt}}$\rightarrow${\BV{qdr}} setting, the performances of our models are worse than source-only baseline, which indicates that negative transfer~\cite{pan2010survey} occurs. (3) In the \textit{source combine} setting, the performances of DAN~\cite{long2015}, RTN~\cite{RTN}, JAN~\cite{JAN}, DANN~\cite{DANN} are lower than the source only baseline, indicating the negative transfer happens when the training data are from multiple source domains.

\noindent \textbf{Effect of Category Number} To show how the number of categories affects the performance of state-of-the-art domain adaptation methods, we choose the \textit{painting}$\rightarrow$\textit{real} setting in \DATASETNAME and gradually increase the number of category from 20 to 345. The results are  in Figure~\ref{fig_relation_of_categories}. An interesting observation is that when the number of categories is small (which is exactly the case in most domain adaptation benchmarks), all methods tend to perform well. However, their performances drop at different rates when the number of categories increases. For example, SE~\cite{SE} performs the best when there is a limit number of categories, but worst when the number of categories is larger than 150.

%% file: table_baseline.tex
\begin{table*}[t]
 \vspace{0.03cm}
\setlength{\tabcolsep}{0.085em}
\renewcommand{\arraystretch}{0.8}
{\scriptsize{
\centering

\begin{tabular}{c|c c c c c c c || c | c c c c c c c || c |c c c c c c c || c | c c c c c c c }
\redbold{AlexNet} & \BV{clp} & \BV{inf} & \BV{pnt} & \BV{qdr} & \BV{rel} & \BV{skt} & \BU{Avg.} &
\redbold{DAN} & \BV{clp} & \BV{inf} & \BV{pnt} & \BV{qdr} & \BV{rel} & \BV{skt} & \BU{Avg.} &
\redbold{JAN} & \BV{clp} & \BV{inf} & \BV{pnt} & \BV{qdr} & \BV{rel} & \BV{skt} & \BU{Avg.} &
\redbold{DANN} & \BV{clp} & \BV{inf} & \BV{pnt} & \BV{qdr} & \BV{rel} & \BV{skt} & \BU{Avg.} \\
\hline

\BV{clp} & 65.5	& 8.2 &	21.4 &	10.5 &	36.1 &	10.8 &	\BU{17.4} & 
\BV{clp} & N/A    & 9.1 &	23.4 &	16.2 &	37.9 &	29.7 &	\BU{23.2}&
\BV{clp} & N/A    & 7.8 &	24.5 &	14.3 &	38.1&	25.7&	\BU{22.1}&
\BV{clp} & N/A&	 9.1 &	23.2&	13.7&	37.6&	28.6&	\BU{22.4}\\

\BV{inf}&	32.9&	27.7&	23.8&	2.2&	26.4&	13.7&	\BU{19.8}&	
\BV{inf}&	17.2&	N/A&	15.6&	4.4&	24.8&	13.5&	\BU{15.1}&	
\BV{inf}&	17.6&	N/A&	18.7   &	8.7&28.1&	15.3&	\BU{17.7}&	
\BV{inf}&	17.9&	N/A&	16.4&	2.1&	27.8&	13.3&	\BU{15.5}\\

\BV{pnt}&	28.1&	7.5&	57.6&	2.6&	41.6&	20.8&	\BU{20.1}&  	
\BV{pnt}&	29.9&	8.9&	N/A&	7.9&42.1&	26.1&	\BU{23.0}&	
\BV{pnt}&	27.5&	8.2	& N/A	&   7.1&	43.1&	23.9&	\BU{22.0}&	
\BV{pnt}&	29.1&	8.6&	N/A&	5.1&	41.5&	24.7&	\BU{21.8} \\

\BV{qdr}&	13.4&	1.2&	2.5&	68.0&	5.5& 7.1&	\BU{5.9}&
\BV{qdr}&	14.2&	1.6&	4.4&	N/A&	8.5&10.1&	\BU{7.8}&	
\BV{qdr}&	17.8&	2.2	&7.4&	N/A&	8.1&	10.9&	\BU{9.3}&	
\BV{qdr}&	16.8&	1.8&	4.8&	N/A&	9.3	&10.2&	\BU{8.6}\\

\BV{rel}&	36.9&	10.2&	33.9&	4.9&	72.8&	23.1&	\BU{21.8}&	
\BV{rel}&	37.4&	11.5&	33.3&	10.1&	N/A&	26.4&	\BU{23.7}&	
\BV{rel}&	33.5&	9.1&	32.5&	7.5	&N/A&	21.9&	\BU{20.9}	&
\BV{rel}&	36.5&	11.4&	33.9&	5.9	&N/A&	24.5&	\BU{22.4}\\

\BV{skt}&	35.5&	7.1&	21.9&	11.8&	30.8&	56.3&	\BU{21.4}&	
\BV{skt}&	39.1&	8.8	&28.2&	13.9	&36.2&	N/A&	\BU{25.2}	&
\BV{skt} &35.3&	8.2	&27.7&	13.3&	36.8&	N/A&	\BU{24.3}&	
\BV{skt} &37.9&	8.2	&26.3&	12.2&	35.3&	N/A&	\BU{24.0}\\

\BU{Avg.}&	\BU{29.4}&	\BU{6.8}&	\BU{20.7}&	\BU{6.4}&	\BU{28.1}&	\BU{15.1}&	\redbold{17.8}&
\BU{Avg.}&	\BU{27.6}&	\BU{8.0}&	\BU{21.0}&	\BU{10.5}&	\BU{29.9}&	\BU{21.2}&	\redbold{19.7}&
\BU{Avg.}&	\BU{26.3}&	\BU{7.1}&	\BU{22.2}&	\BU{10.2}&	\BU{30.8}&	\BU{19.5}&	\redbold{19.4}&
\BU{Avg.}&	\BU{27.6}&	\BU{7.8}&	\BU{20.9}&	\BU{7.8}&	\BU{30.3}&	\BU{20.3}&	\redbold{19.1}\\
\hline\hline

\redbold{RTN} & \BV{clp} & \BV{inf} & \BV{pnt} & \BV{qdr} & \BV{rel} & \BV{skt} & \BU{Avg.} &
\redbold{ADDA} & \BV{clp} & \BV{inf} & \BV{pnt} & \BV{qdr} & \BV{rel} & \BV{skt} & \BU{Avg.} &
\redbold{MCD} & \BV{clp} & \BV{inf} & \BV{pnt} & \BV{qdr} & \BV{rel} & \BV{skt} & \BU{Avg.} &
\redbold{SE} & \BV{clp} & \BV{inf} & \BV{pnt} & \BV{qdr} & \BV{rel} & \BV{skt} & \BU{Avg.} \\

\hline

\BV{clp} & N/A	& 8.1 &	21.1 &	13.1 &	36.1 &	26.5 &	\BU{21.0} & 
\BV{clp} & N/A    & 11.2 &	24.1  &	3.2 &	41.9 &	30.7 &	\BU{22.2}&
\BV{clp} & N/A    & 14.2&	26.1 &	1.6 &	45.0&	33.8&	\BU{24.1}&
\BV{clp} & N/A&	 9.7 &	12.2&	2.2&	33.4&	23.1&	\BU{16.1}\\

\BV{inf}&	15.6&	N/A&	15.3&	3.4&	25.1&	12.8&	\BU{14.4}&	
\BV{inf}&	19.1&	N/A&	16.4&	3.2&	26.9&	14.6&	\BU{16.0}&	
\BV{inf}&	23.6&	N/A&	21.2&	1.5&   36.7&	18.0&	\BU{20.2}&	
\BV{inf}&	10.3&	N/A&	9.6&	1.2&	13.1&	6.9&	\BU{8.2}\\

\BV{pnt}&	26.8&	8.1&	N/A&	5.2&	40.6&	22.6&	\BU{20.7}&  	
\BV{pnt}&	31.2&	9.5&	N/A&	8.4&    39.1&	25.4&	\BU{22.7}&	
\BV{pnt}&	34.4&	14.8	& N/A	&1.9&	50.5&	28.4&	\BU{26.0}&	
\BV{pnt}&	17.1&	9.4&	N/A&	2.1&	28.4&	15.9&	\BU{14.6} \\

\BV{qdr}&	15.1&	1.8&	4.5&	N/A&	8.5& 8.9&	\BU{7.8}&
\BV{qdr}&	15.7&	2.6&	5.4&	N/A&	9.9& 11.9&	\BU{9.1}&	
\BV{qdr}&	15.0&	3.0	&7.0&	N/A&	11.5&	10.2&	\BU{9.3}&	
\BV{qdr}&	13.6&	3.9&	11.6&	N/A&	16.4	&11.5&	\BU{11.4}\\

\BV{rel}&	35.3&	10.7&	31.7&	7.5&	N/A&	22.9&	\BU{21.6}&	
\BV{rel}&	39.5&	14.5&	29.1&	12.1&	N/A&	25.7&	\BU{24.2}&	
\BV{rel}&	42.6&	19.6&	42.6&	2.2	&N/A&	29.3&	\BU{27.2}	&
\BV{rel}&	31.7&	12.9&	19.9&	3.7	&N/A&	26.3&	\BU{18.9}\\

\BV{skt}&	34.1&	7.4&	23.3&	12.6&	32.1&	N/A&	\BU{21.9}&	
\BV{skt}&	35.3&	8.9	& 25.2&	14.9	&37.6&	N/A&	\BU{25.4}	&
\BV{skt} &41.2&	13.7	&27.6&	3.8&	34.8&	N/A&	\BU{24.2}&	
\BV{skt} &18.7&	7.8	&12.2&	7.7&	28.9&	N/A&	\BU{15.1}\\

\BU{Avg.}&	\BU{25.4}&	\BU{7.2}&	\BU{19.2}&	\BU{8.4}&	\BU{28.4}&	\BU{18.7}&	\redbold{17.9}&
\BU{Avg.}&	\BU{28.2}&	\BU{9.3}&	\BU{20.1}&	\BU{8.4}&	\BU{31.1}&	\BU{21.7}&	\redbold{19.8}&
\BU{Avg.}&	\BU{31.4}&	\BU{13.1}&	\BU{24.9}&	\BU{2.2}&	\BU{35.7}&	\BU{23.9}&	\redbold{21.9}&
\BU{Avg.}&	\BU{18.3}&	\BU{8.7}&	\BU{13.1}&	\BU{3.4}&	\BU{24.1}&	\BU{16.7}&	\redbold{14.1}\\

\end{tabular}

}}
\vspace{0.1cm}
\caption{\textbf{Single-source baselines on the \DATASETNAME dataset.} Several single-source adaptation baselines are evaluated on the \DATASETNAME dataset, including AlexNet~\cite{alexnet}, DAN~\cite{long2015}, JAN~\cite{JAN}, DANN~\cite{DANN}, RTN~\cite{RTN}, ADDA~\cite{adda}, MCD~\cite{MCD}, SE~\cite{SE}. In each sub-table, the column-wise domains are selected as the source domain and the row-wise domains are selected as the target domain. The green numbers represent the average performance of each column or row. The red numbers denote the average accuracy for all the 30 (source, target) combinations. }
\label{tab_challengeI}
\vspace{-0.2cm}
\end{table*}

%% file: table_main_result.tex
\begin{table*}
\centering
\footnotesize
{
\begin{tabular}{l|c |c |c| c| c |c | c |c}
\multirow{2}{0.8cm}{Standards} &
\multirow{2}{0.9cm}{Models} & 
\multirow{2}{1.4cm}{\scriptsize{\BV{inf,pnt,qdr,\\rel,skt$\rightarrow$clp}} } &  
\multirow{2}{1.4cm}{\scriptsize{\BV{clp,pnt,qdr,\\rel,skt$\rightarrow$inf}} } & 
\multirow{2}{1.4cm}{\scriptsize{\BV{clp,inf,qdr,\\rel,skt$\rightarrow$pnt}} }& 
\multirow{2}{1.4cm}{\scriptsize{\BV{clp,inf,pnt,\\rel,skt$\rightarrow$qdr}} }&    
\multirow{2}{1.4cm}{\scriptsize{\BV{clp,inf,pnt,\\qdr,skt $\rightarrow$rel}} } & 
\multirow{2}{1.4cm}{\scriptsize{\BV{clp,inf,pnt,\\qdr,rel $\rightarrow$skt}} } & 
\multirow{2}{0.5cm}{{\BU{Avg}}} \\ 
&&&&&&&& \\

 \Xhline{0.7pt} 
\multirow{8}{*}{ \begin{tabular}[c]{@{}c@{}}Single\\Best\end{tabular} } 
& Source Only & 39.6$\pm$0.58 & 8.2$\pm$0.75 & 33.9 $\pm$ 0.62 & 11.8 $\pm$ 0.69 & 41.6 $\pm$ 0.84 & 23.1$\pm$0.72 & \BU{26.4 }$\pm$ 0.70\\

 &DAN~\cite{long2015} &  39.1$\pm$0.51 & 11.4$\pm$0.81 & 33.3$\pm$0.62 & \redbold{16.2}$\pm$0.38& 42.1$\pm$0.73  & 29.7$\pm$0.93 & \BU{28.6}$\pm$0.63\\
 &RTN~\cite{RTN} &  35.3$\pm$0.73 & 10.7$\pm$0.61 & 31.7$\pm$0.82 & 13.1$\pm$0.68& 40.6$\pm$0.55  & 26.5$\pm$0.78 & \BU{26.3}$\pm$0.70\\
  &JAN~\cite{JAN} &  35.3$\pm$0.71 & 9.1$\pm$0.63 & 32.5$\pm$0.65 & 14.3$\pm$0.62& 43.1$\pm$0.78  & 25.7$\pm$0.61 & \BU{26.7}$\pm$0.67\\
   &DANN~\cite{DANN} &  37.9$\pm$0.69 & 11.4$\pm$0.91 & 33.9$\pm$0.60 & 13.7$\pm$0.56& 41.5$\pm$0.67  & 28.6$\pm$0.63 & \BU{27.8}$\pm$0.68\\
    &ADDA~\cite{adda} &  39.5$\pm$0.81 & 14.5$\pm$0.69 & 29.1$\pm$0.78 & 14.9$\pm$0.54& 41.9$\pm$0.82  & 30.7$\pm$0.68 & \BU{28.4}$\pm$0.72\\
 &SE~\cite{SE} & 31.7$\pm$0.70 &  12.9$\pm$0.58  & 19.9$\pm$0.75 & 7.7$\pm$0.44 &33.4$\pm$0.56&26.3$\pm$0.50&\BU{22.0}$\pm$0.66 \\ 
 &MCD~\cite{MCD} & 42.6$\pm$0.32 &  19.6$\pm$0.76  & 42.6$\pm$0.98 & 3.8$\pm$0.64 &50.5$\pm$0.43&33.8$\pm$0.89&\BU{32.2}$\pm$0.66 \\

 \Xhline{0.7pt} 
\multirow{8}{*}{ \begin{tabular}[c]{@{}c@{}}Source\\Combine\end{tabular} } 
&Source Only & 47.6$\pm$0.52  & 13.0$\pm$0.41 & 38.1$\pm$0.45   & 13.3$\pm$0.39 & 51.9$\pm$0.85 & 33.7$\pm$0.54 & \BU{32.9}$\pm$0.54\\
&DAN~\cite{long2015}& 45.4$\pm$0.49&	12.8$\pm$0.86&	36.2$\pm$0.58&	15.3$\pm$0.37&	48.6$\pm$0.72&	34.0$\pm$0.54&	\BU{32.1}$\pm$0.59  \\
 &RTN~\cite{RTN} & 44.2$\pm$0.57&	12.6$\pm$0.73&	35.3$\pm$0.59&	14.6$\pm$0.76&	48.4$\pm$0.67&	31.7$\pm$0.73&	\BU{31.1}$\pm$0.68\\
&JAN~\cite{JAN}& 40.9$\pm$0.43&	11.1$\pm$0.61&	35.4$\pm$0.50&	12.1$\pm$0.67&	45.8$\pm$0.59&	32.3$\pm$0.63&	\BU{29.6}$\pm$0.57  \\
&DANN~\cite{DANN}& 45.5$\pm$0.59&	13.1$\pm$0.72&	37.0$\pm$0.69&	13.2$\pm$0.77&	48.9$\pm$0.65&	31.8$\pm$0.62&	\BU{32.6}$\pm$0.68  \\
&ADDA~\cite{adda}& 47.5$\pm$0.76&	11.4$\pm$0.67&	36.7$\pm$0.53&	14.7$\pm$0.50&	49.1$\pm$0.82&	33.5$\pm$0.49&	\BU{32.2}$\pm$0.63  \\

&SE~\cite{SE}& 24.7$\pm$0.32&	3.9$\pm$0.47&	12.7$\pm$0.35&	 7.1$\pm$0.46&	22.8$\pm$0.51&	9.1$\pm$0.49&	\BU{16.1}$\pm$0.43  \\

&MCD~\cite{MCD}& 54.3$\pm$0.64&	22.1$\pm$0.70&	45.7$\pm$0.63&	7.6$\pm$0.49&	58.4$\pm$0.65&	43.5$\pm$0.57&	\BU{38.5}$\pm$0.61  \\

\Xhline{0.7pt}
\multirow{4}{*}{ \begin{tabular}[c]{@{}c@{}}Multi-\\Source\end{tabular} } 
&DCTN~\cite{xu2018deep} &48.6$\pm$0.73  & 23.5$\pm$0.59  &48.8$\pm$0.63  &7.2$\pm$0.46& 53.5$\pm$0.56 & 47.3$\pm$0.47 & \BU{38.2}$\pm$0.57 \\

&\textbf{\NameNoSpace}$^{\ast}$~(ours) & 57.0$\pm$0.79 & 22.1$\pm$0.68  &50.5$\pm$0.45  &4.4$\pm$ 0.21& 62.0$\pm$0.45 & 48.5$\pm$0.56 & \BU{40.8}$\pm$ 0.52\\

&\textbf{\NameNoSpace}~(ours)&57.2$\pm$0.98  & 24.2$\pm$1.21 & 51.6$\pm$0.44 &5.2$\pm$0.45  & 61.6$\pm$0.89& \redbold{49.6}$\pm$0.56 &  \BU{41.5}$\pm$0.74\\

& \textbf{\NameBNoSpace}~(ours)&\redbold{58.6}$\pm$0.53& \redbold{26.0}$\pm$ 0.89& \redbold{52.3}$\pm$0.55& 6.3$\pm$0.58& \redbold{62.7}$\pm$0.51& 49.5$\pm$0.76& \redbold{42.6}$\pm$0.64 \\

\Xhline{0.7pt}
\multirow{3}{*}{ \begin{tabular}[c]{@{}c@{}}Oracle \\Results\end{tabular} } 

&AlexNet &65.5$\pm$0.56  & 27.7$\pm$0.34  &57.6$\pm$0.49 &68.0$\pm$0.55& 72.8$\pm$0.67 & 56.3$\pm$0.59 & \BU{58.0}$\pm$0.53 \\

&ResNet101 &69.3$\pm$0.37  & 34.5$\pm$0.42  &66.3$\pm$0.67 &66.8$\pm$0.51& 80.1$\pm$0.59 & 60.7$\pm$0.48 & \BU{63.0}$\pm$0.51 \\

&ResNet152 &71.0$\pm$0.63  & 36.1$\pm$0.61  &68.1 $\pm$ 0.49  &69.1$\pm$0.52 & 81.3$\pm$0.49 & 65.2$\pm$0.57 & \BU{65.1}$\pm$0.55 
                                  
\end{tabular}
} 
\vspace{0.1cm}
\caption{
\textbf{Multi-source domain adaptation results on the \DATASETNAMENoSpace~ dataset.} Our model \Name and \NameB achieves \textbf{41.5\%} and \textbf{42.6\%} accuracy, significantly outperforming all other baselines. \NameNoSpace$^{\ast}$ indicates the normal average of all the classifiers. When the target domain is \textit{quickdraw}, the multi-source methods perform worse than single-source and source only baselines, which indicates negative transfer~\cite{pan2010survey} occurs in this case. ({\BV{clp}}: \textit{clipart}, {\BV{inf}}: \textit{infograph}, {\BV{pnt}}: \textit{painting}, {\BV{qdr}}: \textit{quickdraw}, {\BV{rel}}: \textit{real}, {\BV{skt}}: \textit{sketch}.)
}
\label{table_lsdac}
\vspace{-0.1cm}
\end{table*}

%% file: 6_conclusion.tex
\vspace{-.55em}
\section{Conclusion}
In this paper, we have collected, annotated and evaluated by far the largest domain adaptation dataset named DomainNet. The dataset is challenging due to the presence of notable domain gaps and a large number of categories. We hope it will be beneficial to evaluate future single- and multi-source UDA methods. 

We have also proposed \Name to align multiple source domains with the target domain. We derive a meaningful error bound for our method under the framework of cross-moment divergence. Further, we incorporate the moment matching component into deep neural network and train the model in an end-to-end fashion. Extensive experiments on multi-source domain adaptation benchmarks demonstrate that our model outperforms all the multi-source baselines as well as the best single-source domain adaptation method. 

\section{Acknowledgements}
We thank Ruiqi Gao, Yizhe Zhu, Saito Kuniaki, Ben Usman, Ping Hu for their useful discussions and suggestions. We thank anonymous annotators for their hard work to label the data. This work was partially supported by NSF and Honda Research Institute. The authors also acknowledge support from CIFAR AI Chairs Program.

%% file: 7_supp.tex
\clearpage

\section{Appendix}
\addcontentsline{toc}{section}{Appendices}
\renewcommand{\thesubsection}{\Alph{subsection}}

The appendix is organized as follows: Section~\ref{sec_ablation_LSDAC} shows the ablation study for source-source alignment. Section~\ref{sec_supp_cross_moment} introduces the formal definition of the cross-moment divergence; Section~\ref{sec_proof} gives the proof of Theorem~\ref{thm:pairwise_bound} and further discussions; Section~\ref{sec_digit_exp} provides the details of experiments on ``Digit-Five" dataset; Section~\ref{sec_supp_feature_visualization} shows feature visualization with t-SNE plot; Section~\ref{sec_supp_category_number} shows how the number of categories will affect the performance of the state-of-the-art models; Section~\ref{sec_supp_resnet_baselinese} and Section~\ref{sec_supp_train_test_split} introduce the ResNet baselines and Train/Test split of our \DATASETNAME dataset, respectively; Section~\ref{sec_supp_image_sample} and Section~\ref{sec_supp_dataset_statistics} show the image samples and the detailed statistics of our \DATASETNAME dataset; Section~\ref{toy_exp} shows a toy experiment to demonstrate the importance of aligning the source domains; Section~\ref{time_consumption} shows the time consumption of our method, compared to baseline.

\subsection{Ablation Study}
\label{sec_ablation_LSDAC}

To show how much performance gain we can get through source-source alignment (S-S) and source-target (S-T) alignment, we perform ablation study based on our model. From Table~\ref{table_abalation}, we observe the key factor to the performance boost is matching the moments of source distributions to the target distribution. Matching source domains with each other further boosts the performance. The experimental results empirically demonstrate that aligning source domains is essential for MSDA.
 
 \begin{table}[h]
\centering
\small
{
\vspace{-0.2cm}
\begin{tabular}{c |c |c| c}

{Schema} & 
{\scriptsize{digit-five }} &  
{\scriptsize{Office-Caltech10}}&
{\scriptsize{\DATASETNAMENoSpace}}\\ 

\hline
S-S only&81.5 (+4.1) & 94.5 (+1.6)  &34.4 (+1.5)  \\
S-T only& 85.8 (\textbf{+8.1})& 96.2 (\textbf{+3.3})& 39.7 (\textbf{+6.8}) \\
\NameB & 87.7 (+10) & 96.4 (+3.5) & 42.6 (+9.7)\\
                                  
\end{tabular}
} 

\vspace{0.1cm}
\caption{S-S only: only matching source domains with each other; S-T only: only matching source  with target; ``+'': performance gain from baseline.}

\label{table_abalation}
\vspace{-0.2cm}
\end{table}

\subsection{Cross-moment Divergence}
\label{sec_supp_cross_moment}

\begin{definition}[cross-moment divergence]
\label{def:cmomentdiv}
Given a compact domain $\gX\subset\R^n$ and two probability measures $\mu,\mu'$ on $\gX$, the $k$-th order cross-moment divergence between $\mu$ and $\mu'$ is
\begin{equation*}
\begin{split}
    & \ d_{CM^k}(\mu,\mu')\\
   =& \sum_{\rvi\in\Delta_k}
    	\Bigg|\int_{\gX}\prod_{j=1}^n(x_j)^{i_j}d\mu(\vx) 
      -\int_{\gX}\prod_{j=1}^n(x_j)^{i_j}d\mu'(\vx)\Bigg|,
\end{split}
\end{equation*}
where $\Delta_k=\{(i_1,i_2,\ldots,i_n)\in\N^n_0|\sum_{j=1}^ni_j=k\}$.
\end{definition} 

As seen in the rest of the paper, for two domains $D=(\mu, f)$ and $D'=(\mu',f')$, we use $d_{CM^k}(D,D')$ to denote $d_{CM^k}(\mu,\mu')$ for readability concerns.

\input{Appendix/1_proof_of_thm1.tex}


\begin{table*}[t]
\vspace{-0.1cm}
\setlength{\tabcolsep}{0.085em}
\renewcommand{\arraystretch}{0.8}
{
\centering

\begin{tabular}{c|c c c c c c c || c | c c c c c c c }
\hline
\redbold{ResNet101} & \BV{clp} & \BV{inf} & \BV{pnt} & \BV{qdr} & \BV{rel} & \BV{skt} & \BU{Avg.} &
\redbold{ResNet152} & \BV{clp} & \BV{inf} & \BV{pnt} & \BV{qdr} & \BV{rel} & \BV{skt} & \BU{Avg.} \\
\hline

\BV{clp} &N/A &19.3 &37.5 &11.1 &52.2 &41.0 &	\BU{32.2} & 
\BV{clp} & N/A  &  19.8 &37.9 &12.2 &52.3 &44.8 &	\BU{33.4}\\

\BV{inf}&30.2&N/A &31.2 &3.6 &44.0 &27.9 &	\BU{27.4}&	
\BV{inf}&	31.3&N/A &31.1 &4.7 &45.5 &29.6 &	\BU{28.4}\\

\BV{pnt}&	39.6 &18.7 &N/A&4.9 &54.5 &36.3 &	\BU{30.8}&  	
\BV{pnt}&	42.0 &19.5&N/A &7.4 &55.0 &37.7 &	\BU{32.3} \\

\BV{qdr}&7.0 &0.9 &1.4 &N/A&4.1 &8.3 & 	\BU{4.3}&
\BV{qdr}&	12.2 &1.8 &2.9 &N/A&6.3 &9.4 &	\BU{6.5}\\

\BV{rel}&48.4 &22.2 &49.4 &6.4&N/A &38.8 & 	\BU{33.0}&	
\BV{rel}&	50.5 &24.4 &49.0 &6.2&N/A&39.9 & \BU{34.0}\\

\BV{skt}&46.9 &15.4 &37.0 &10.9 &47.0 &N/A&	\BU{31.4}&	
\BV{skt}&	51.0 &18.2 &39.7 &12.5 &47.4 &N/A&	\BU{33.8}\\

\BU{Avg.}&	\BU{34.4}&	\BU{15.3}&	\BU{31.3}&	\BU{7.4}&	\BU{40.4}&	\BU{30.5}&	\redbold{26.5}&
\BU{Avg.}&	\BU{37.4}&	\BU{16.7}&	\BU{32.1}&	\BU{8.6}&	\BU{41.3}&	\BU{32.3}&	\redbold{28.1}\\
\hline

\end{tabular}

}
\vspace{0.1cm}
\caption{\textbf{Single-source \textit{ResNet101} and \textit{ResNet152}~\cite{he2015deep} baselines on the \DATASETNAME dataset.} We provide ResNet baselines for Table~\ref{tab_challengeI}. In each sub-table, the column-wise domains are selected as the source domain and the row-wise domains are selected as the target domain. The green numbers represent the average performance of each column or row. The red numbers denote the average accuracy for all the 30 (source, target) combinations. ({\BV{clp}}: \textit{clipart}, {\BV{inf}}: \textit{infograph}, {\BV{pnt}}: \textit{painting}, {\BV{qdr}}: \textit{quickdraw}, {\BV{rel}}: \textit{real}, {\BV{skt}}: \textit{sketch}.) }
\label{tab_resnet_source_only}
\vspace{-0.3cm}
\end{table*}

\begin{figure*}
    \centering
    
    \begin{minipage}{\hsize}
      \centering
      \subfigure[ \textit{painting}$\rightarrow$\textit{real} ]
      {\includegraphics[width=0.46\hsize]{images/num_cat_2_acc.pdf}
      \label{fig_supp_painting_real} }
     \centering
      \subfigure[\textit{infograph}$\rightarrow$\textit{real}]
      {\includegraphics[width=0.46\hsize]{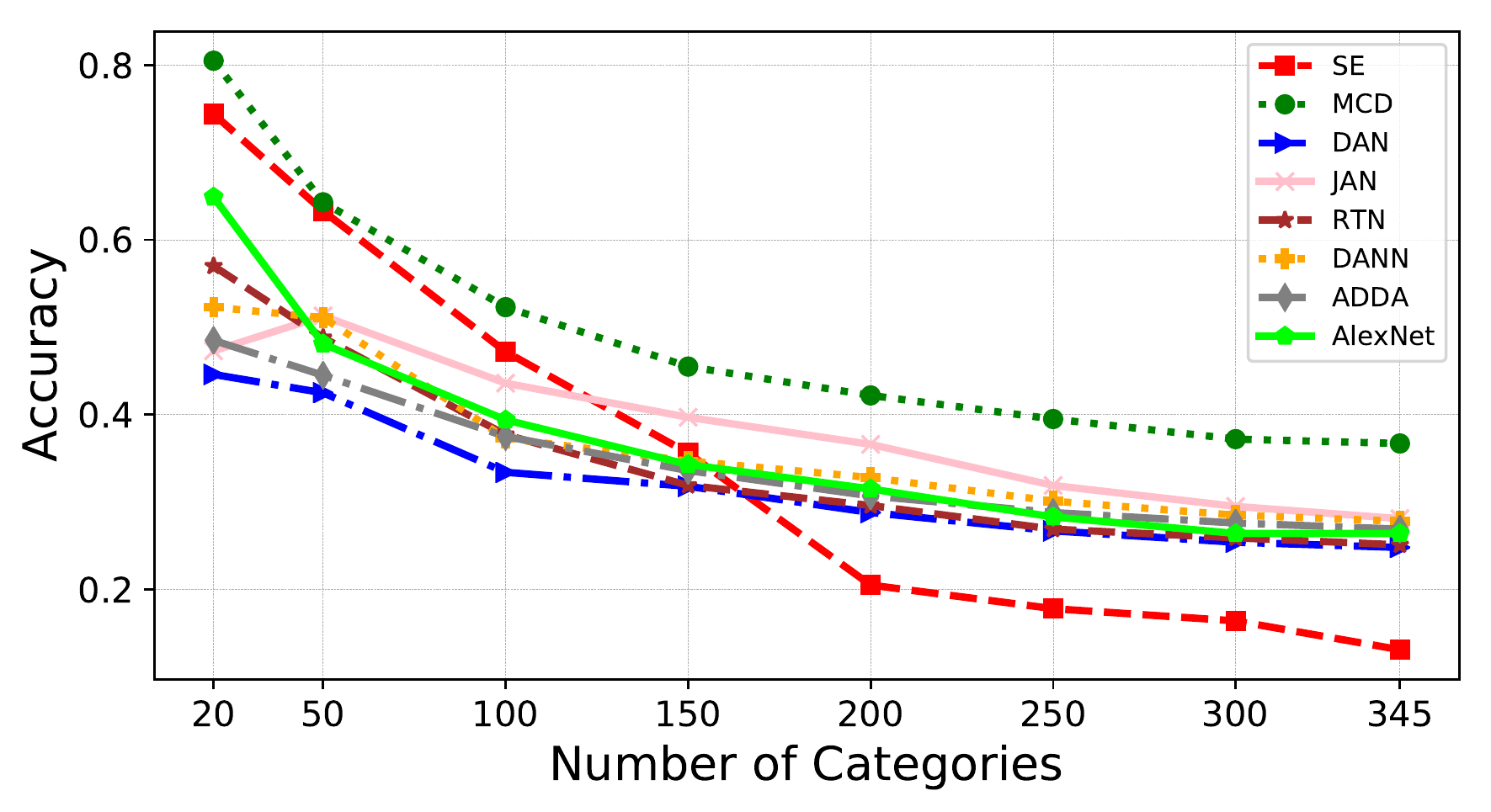}
      \label{fig_supp_info_rel}}
      \centering
      \subfigure[\textit{sketch}$\rightarrow$\textit{clipart}]
      {\includegraphics[width=0.46\hsize,height=4.5cm]{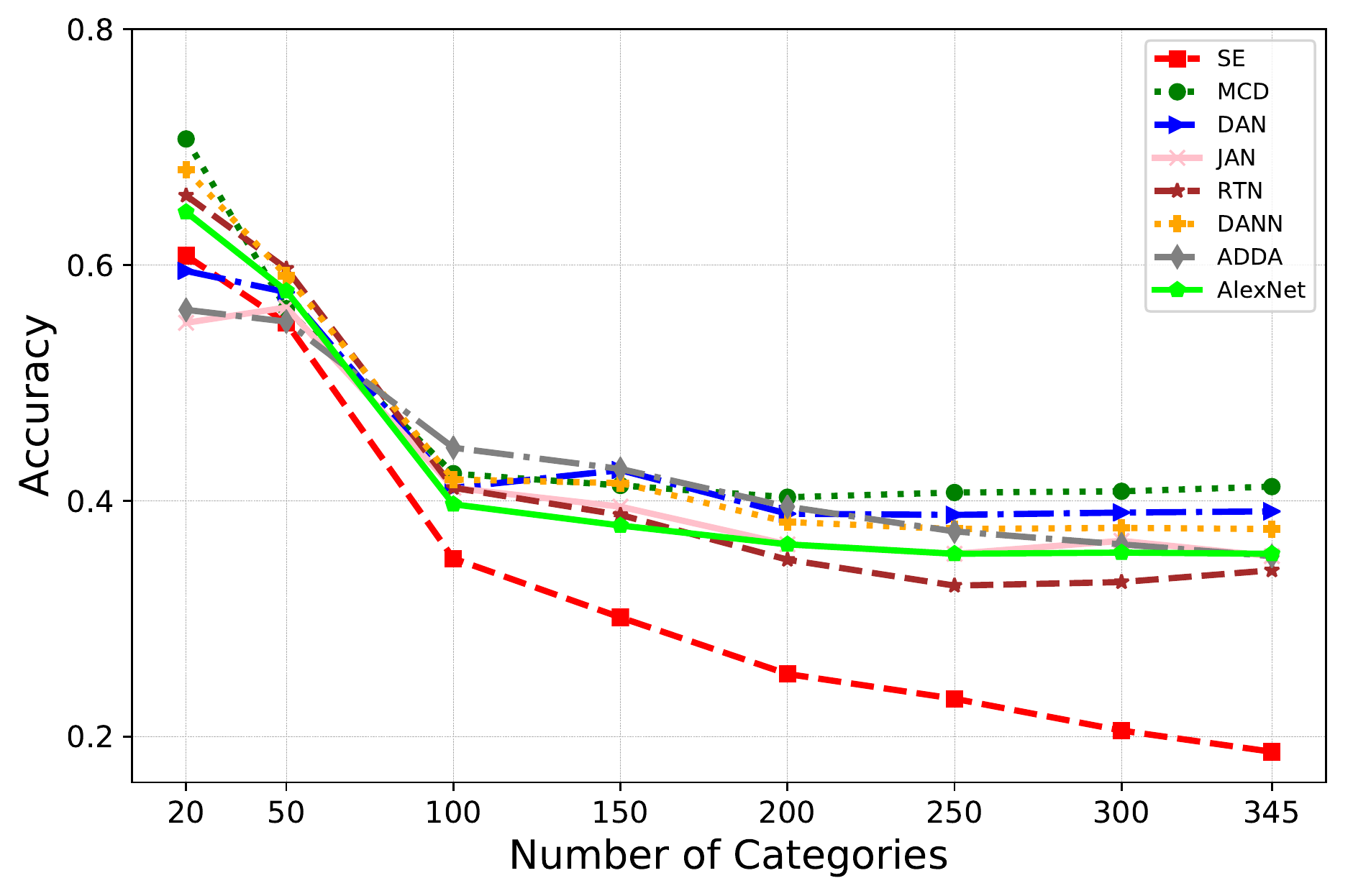} 
      \label{fig_supp_skt_clp}}
      \centering
      \subfigure[\textit{quickdraw}$\rightarrow$\textit{clipart}]
      { \includegraphics[width=0.46\hsize,height=4.7cm]{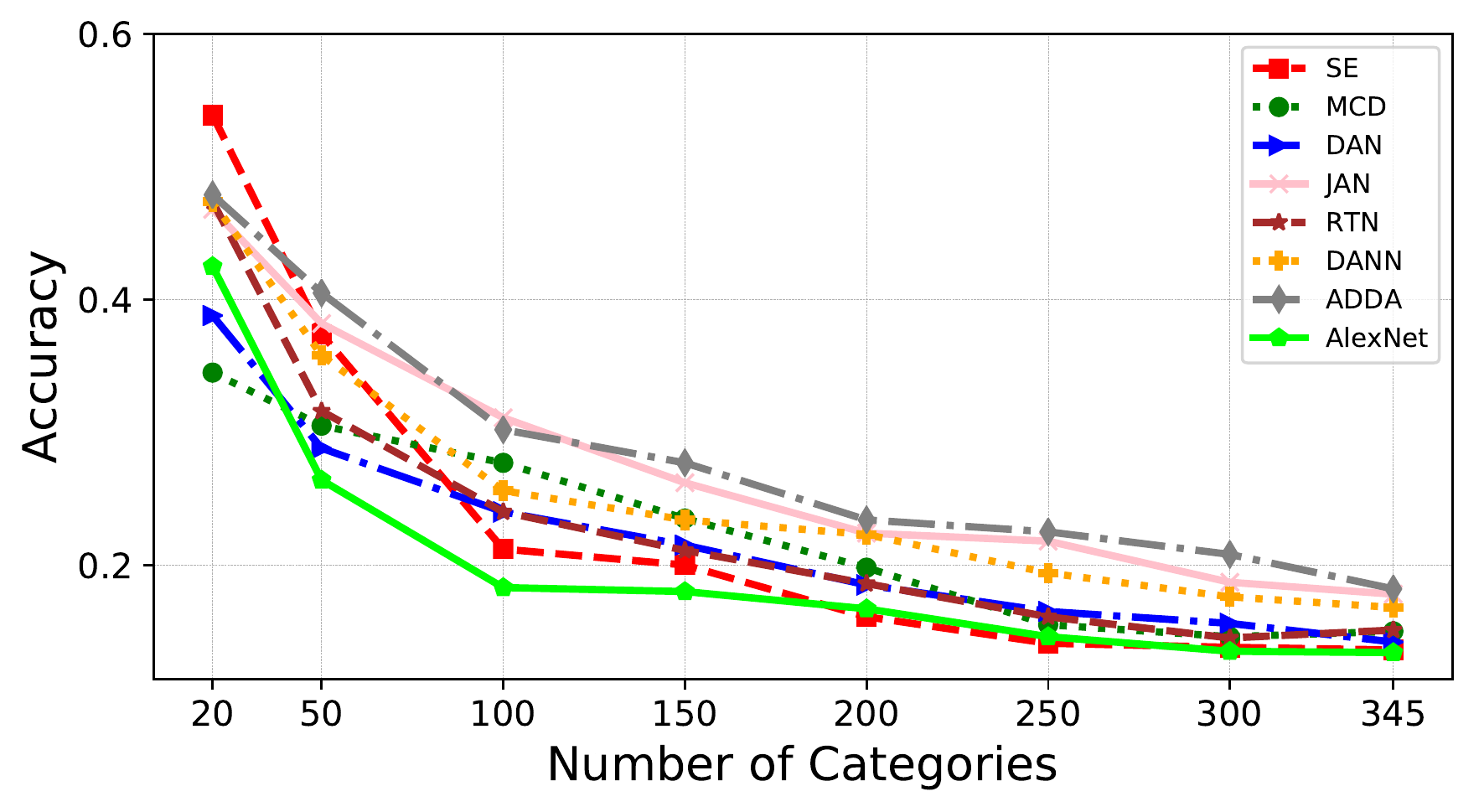}
      \label{fig_supp_qdr_clp}}
    \end{minipage}
    
    \caption{\textbf{Accuracy \textit{vs.} Number of Catogries}. We plot how the performances of different models will change when the number of categories increases. We select four UDA settings, \ie, \textit{painting}$\rightarrow$\textit{real}, \textit{infograph}$\rightarrow$\textit{real},
    \textit{sketch}$\rightarrow$\textit{clipart},
    \textit{quickdraw}$\rightarrow$\textit{clipart}. The figure shows the performance of all models drop significantly with the increase of category number.
    }
    \label{fig_plot_acc_vs_num}
    \vspace{-0.2cm}
\end{figure*}

\begin{figure*}
\vspace{-0.4cm}
    \begin{minipage}{\hsize}
      \centering
      \subfigure[\scriptsize DAN Features on SVHN ]
      {\includegraphics[width=0.23\hsize]{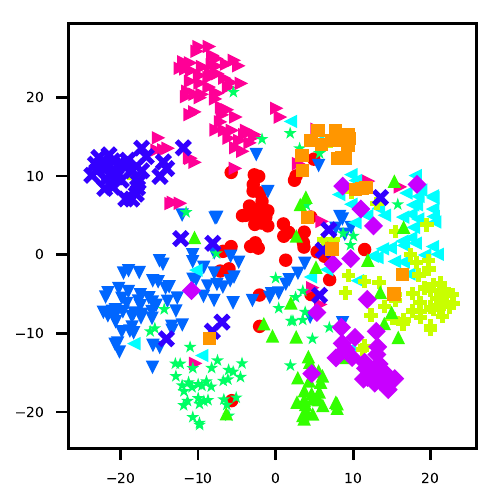}
      \label{mnist_dan} }
     \centering
      \subfigure[\scriptsize \NameB Features on SVHN]
      {\includegraphics[width=0.23\hsize]{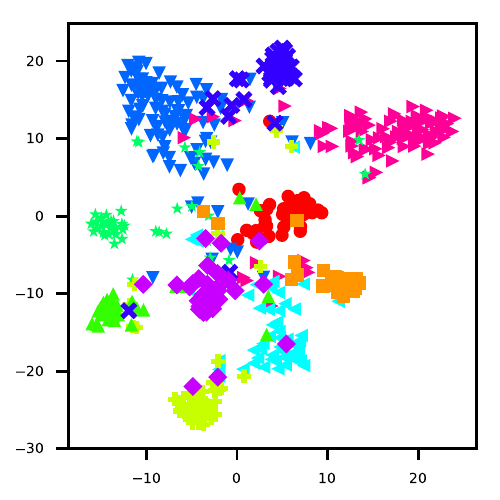}
      \label{mnist_msda}}
      \centering
      \subfigure[\scriptsize DAN Features on Caltech]
      {\includegraphics[width=0.23\hsize]{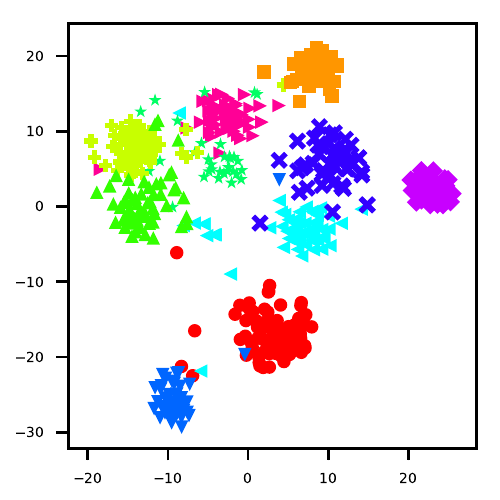} 
      \label{office_dan}}
      \centering
      \subfigure[\scriptsize \NameB Features on Caltech]
      { \includegraphics[width=0.23\hsize]{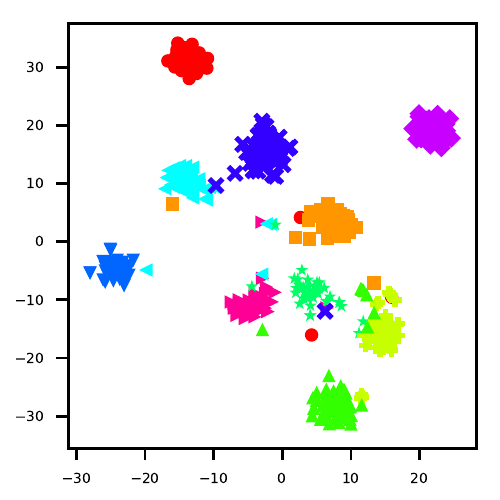}
      \label{office_msda}}
    \end{minipage}
    \vspace{-0.2cm}
  \caption{Feature visualization: t-SNE plot of DAN features and \NameB features on SVHN in mm,mt,up,sy$\rightarrow$sv setting; t-SNE of DAN features and \NameB features on Caltech in A,D,W$\rightarrow$C setting. We use different markers and different colors to denote different categories. (Best viewed in color.)}
  \label{fig_analysis}
  \vspace{-0.4cm}
\end{figure*}

\input{Appendix/4_digit_exp_conf.tex}

\input{Appendix/3_table_main_result_full.tex}

\subsection{Feature visualization} 
\label{sec_supp_feature_visualization}
To demonstrate the transfer ability of our model, we visualize the DAN~\cite{long2015} features and \NameB features with t-SNE embedding in two tasks: mm,mt,up,sy$\rightarrow$sv and A,D,W$\rightarrow$C. The results are shown in Figure~\ref{fig_analysis}. We make two important observations: i) Comparing Figure~\ref{mnist_dan} with Figure~\ref{mnist_msda}, we find that \NameB is capable of learning more discriminative features; ii) From Figure~\ref{office_dan} and Figure~\ref{office_msda}, we find that the clusters of \NameB features are more compact than those of DAN, which suggests that the features learned by \NameB attain more desirable discriminative property. These observations imply the superiority of our model over DAN in multi-source domain adaptation.

\subsection{Effect of Category Number}
\label{sec_supp_category_number}
In this section, we show more results to clarify how the number of categories affects the performances of the state-of-the-art models. 
We choose the following four settings, \ie, \textit{painting}$\rightarrow$\textit{real} (Figure~\ref{fig_supp_painting_real}), \textit{infograph}$\rightarrow$\textit{real} (Figure~\ref{fig_supp_info_rel}), \textit{sketch}$\rightarrow$\textit{clipart} (Figure~\ref{fig_supp_skt_clp}), \textit{quickdraw}$\rightarrow$\textit{clipart} (Figure~\ref{fig_supp_qdr_clp}), and gradually increase the number of categories from 20 to 345. From the four figures, we make the following interesting observations:
\begin{itemize}
    \item All the models perform well when the number of categories is small. However, their performances drop rapidly when the number of categories increases.
    \item Self-Ensembling~\cite{SE} model has a good performance when the number of categories is small, but it is not suitable to large scale domain adaptation. 
    
\end{itemize}

\input{Appendix/6_resnet_table.tex}

\subsection{Image Samples}
\label{sec_supp_image_sample}
We sample the images for each domain and show them in Figure~\ref{fig_clipart} (\textit{clipart}), Figure~\ref{fig_infograph}~(\textit{infograph}), Figure~\ref{fig_painting}~(\textit{painting}), Figure~\ref{fig_quickdraw}~(\textit{quickdraw}), Figure~\ref{fig_real}~(\textit{real}), and Figure~\ref{fig_sketch}~(\textit{sketch}).

\begin{figure}[t]
    \begin{minipage}[t]{\hsize}
            
      \centering
      \subfigure[\scriptsize Original]{\includegraphics[width=0.33\hsize]{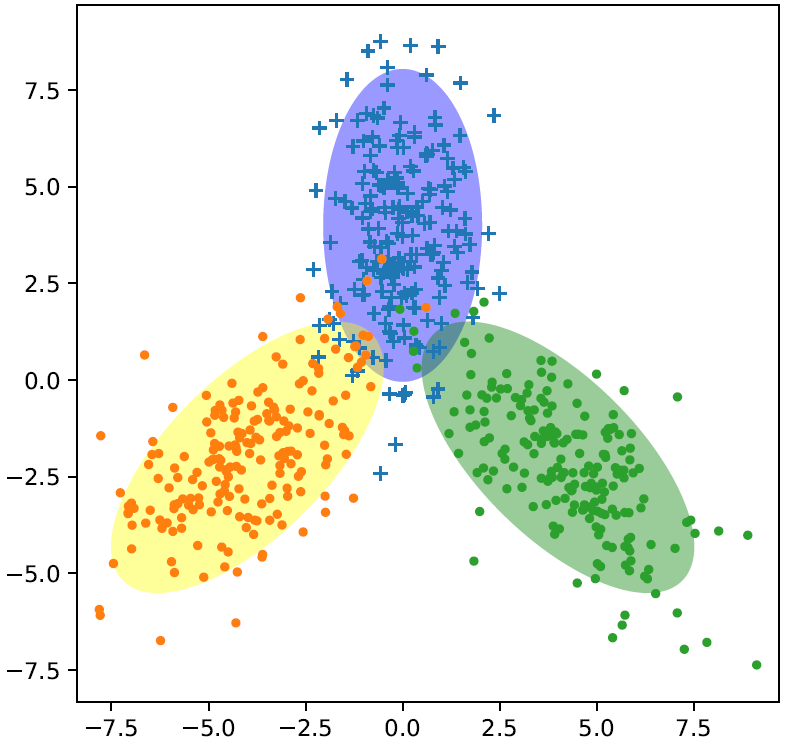}
      \label{toy_fig_1}}
      \centering
      \subfigure[{\scriptsize No Within-Source Matching}]{\includegraphics[width=0.3\hsize]{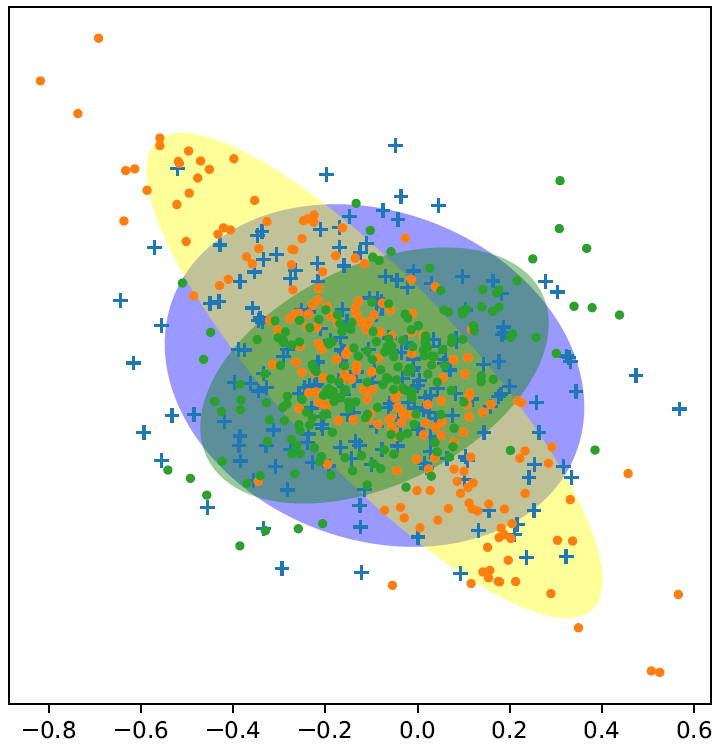}
      \label{toy_fig_2}}
      \centering
      \subfigure[\scriptsize Our model]{\includegraphics[width=0.315\hsize]{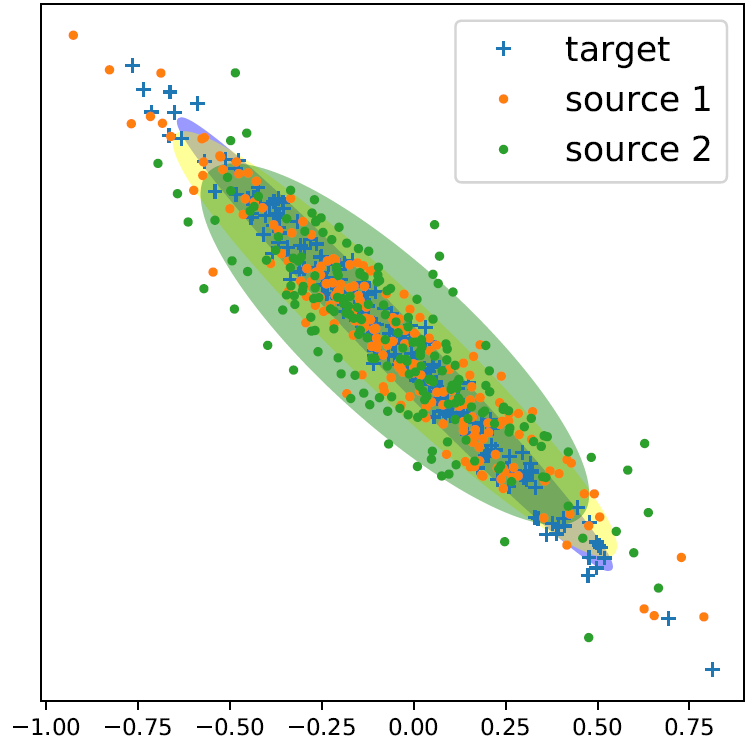}
      \label{toy_fig_3}}
    \end{minipage}
 \caption{Yellow and green points are sampled from two source domains. Blue points are sampled from target domain.
 }
  \label{fig_toy_exp}
  \vspace{-0.3cm}
    \end{figure}

\input{Appendix/y_6_figures.tex}

\subsection{Dataset Statistics}
\label{sec_supp_dataset_statistics}
Table~\ref{dataset_1}, Table~\ref{dataset_2}, and Table~\ref{dataset_3} show the detailed statistics of our \DATASETNAME dataset. Our dataset contains 6 distinct domains, 345 categories and $\sim$0.6 million images. The categories are from 24 divisions, which are: \textit{Furniture}, \textit{Mammal}, \textit{Tool}, \textit{Cloth}, \textit{Electricity}, \textit{Building}, \textit{Office}, \textit{Human Body}, \textit{Road Transportation}, \textit{Food}, \textit{Nature}, \textit{Cold Blooded}, \textit{Music}, \textit{Fruit}, \textit{Sport}, \textit{Tree}, \textit{Bird}, \textit{Vegetable}, \textit{Shape}, \textit{Kitchen}, \textit{Water Transportation}, \textit{Sky Transportation}, \textit{Insect}, \textit{Others}.

\subsection{Toy Experiment}
\label{toy_exp}
In multi-source domain adaptation, the source domains are not \textit{i.i.d} and as a result, are not automatically aligned with each other. Intuitively, it is not possible to perfectly align the target domain with every source domain, if the source domains are not aligned themselves. More specifically, as explained in the paragraph following Theorem ~\ref{thm:pairwise_bound}, the last term of our target error bound is lower bounded by pairwise divergences between source domains. This motivates our algorithm to also align the moments between each pair of source domains. 

Empirically, let's consider a toy experiment setting in which we have two source domains (denoted by yellow and green) and one target domain (denoted by blue), as shown in Figure~\ref{fig_toy_exp}. In the experiments, we utilize a 2-layer fully connected neural network as the backbone. We can observe the source domains and target domain are better aligned when matching source domain distributions using our model is applied. 

\subsection{Time Consumption}
\label{time_consumption}
We show the training and testing time consumption in Figure~\ref{tab_time_consumption}. The experiments are run in PyTorch with a NVIDIA TITAN X GPU on CentOS 7 server. The server has 8 Intel CORE i7 processors. We benchmark all the models with a minibatch size of 16 and an image size of 224 x 224. The CUDA version is 8.0 with CuDNN 5.15.

\begin{table}[]
    \centering
    \begin{tabular}{|c|c|c|}
         \hline
         Models& Training (ms)&Testing (ms)  \\
         \hline
         ResNet101&200.87 &60.84 \\
         \hline
         \Name & 267.58 & 61.20\\
         \hline
    \end{tabular}
    \caption{Time consumption of our model and the baseline.}
    \label{tab_time_consumption}
\end{table}

\input{Appendix/z_table_11_12_13.tex}

%% file: Appendix/1_proof_of_thm1.tex
\subsection{Proof of Theorem~\ref{thm:pairwise_bound}}
\label{sec_proof}

\begin{theorem}[Weierstrass Approximation Theorem] 
\label{lem:weierstrass}
Let $f: \gC \rightarrow\R$ be continuous, where $\gC$ is a compact subset of $\R^n$. There exists a sequence of real polynomials ${(P_m(\vx))}_{m\in\N}$, such that 
\begin{equation*}
\sup_{\vx\in\gC} |f(\vx) - P_{m}(\vx)| \rightarrow 0,\quad as \: m \rightarrow \infty.
\end{equation*}
\end{theorem}
Note that for $\vx\in\R^n$, a multivariate polynomial $P_m:\R^n\to\R$ is of the form
\[P_m(\vx)=\sum_{k=1}^m\sum_{\rvi\in\Delta_k}a_{\rvi}\prod\limits_{j=1}^n(x_j)^{i_j},\]
where $\Delta_k=\{(i_1,i_2,\ldots,i_n)\in\N_0^n\big|\sum_{j=1}^ni_j=k\}.$

\begin{lemma}
\label{lem:keymoment}
For any hypothesis $h,h'\in\gH$, for any $\epsilon>0$, there exist 
an integer $n_\epsilon$ and a constant $a_{n_\epsilon}$, 
such that
\[|\epsilon_S(h,h')-\epsilon_T(h,h')|
\leq \frac{1}{2}a_{n_\epsilon}
\sum_{k=1}^{n_\epsilon}d_{CM^{k}}\big(\gD_S, \gD_T\big)
+ \epsilon \nonumber.\]
\end{lemma}

\begin{proof}
\begin{align}
\label{eqn:lem12a}
 & \ |\epsilon_S(h,h')-\epsilon_T(h,h')|
     \leq \sup_{h,h'\in\gH}
       \big|\epsilon_S(h,h')-\epsilon_T(h,h')\big| \nonumber\\
=& \sup_{h,h'\in\gH}\big|\rmP_{\vx\sim D_S}[h(\vx)\neq h'(\vx)]
  - \rmP_{\vx\sim D_T}[h(\vx)\neq h'(\vx)]\big| \nonumber\\
=& \sup_{h,h'\in\gH}\left|\int_{\gX}\1_{h(\vx)\neq h'(\vx)}d\mu_S
  - \int_{\gX}\1_{h(\vx)\neq h'(\vx)}d\mu_T\right|,
\end{align}
where $\gX$ is a compact subset of $\R^n$. 
For any fixed $h,h'$, the indicator function $\1_{h(\vx)\neq h'(\vx)}(\vx)$ 
is a Lebesgue integrable function ($L^1$ function) on $\gX$.
It is known that the space of continuous functions with compact support, denoted by $\gC_c(\gX)$, is dense in $L^1(\gX)$, 
i.e., any $L^1$ function on $\gX$ can be approximated arbitrarily well\footnote{with respect to the corresponding norm} by functions in $\gC_c(\gX)$. As a result, for any $\frac{\epsilon}{2}>0$, there exists $f\in\gC_c(\gX)$, such that,
\begin{align}
\label{eqn:lem12b}
& \sup_{h,h'\in\gH}\left|\int_{\gX}\1_{h(\vx)\neq h'(\vx)}d\mu_S
  - \int_{\gX}\1_{h(\vx)\neq h'(\vx)}d\mu_T\right| \nonumber\\
\leq&\ \left|\int_{\gX}f(\vx) d\mu_S - \int_{\gX}f(\vx) d\mu_T\right| + \frac{\epsilon}{2}.
\end{align}
Using Theorem~\ref{lem:weierstrass}, for any $\frac{\epsilon}{2}$, there exists a polynomial $P_{n_\epsilon}(\vx)=\sum\limits_{k=1}^{n_\epsilon}\sum\limits_{\rvi\in\Delta_k}\alpha_{\rvi}\prod\limits_{j=1}^n(x_j)^{i_j}$, such that
\begin{align}
\label{eqn:lem12c}
    &\ \left|\int_{\gX}f(\vx) d\mu_S - \int_{\gX}f(\vx)d\mu_T\right|\nonumber\\
\leq&\ \left|\int_{\gX}P_{n_\epsilon}(\vx)d\mu_S 
       - \int_{\gX}P_{n_\epsilon}(\vx)d\mu_T\right| 
       + \frac{\epsilon}{2} \nonumber\\
\leq&\ \sum_{k=1}^{n_\epsilon}\bigg|\sum\limits_{\rvi\in\Delta_{k}}
       a_{\rvi}\int_{\gX}\prod\limits_{j=1}^n(x_j)^{i_j}d\mu_S \nonumber\\
    &\ - \sum\limits_{\rvi\in\Delta_{k}}a_{\rvi}
       \int_{\gX}\prod\limits_{j=1}^n(x_j)^{i_j}d\mu_T\bigg| 
       + \frac{\epsilon}{2} \nonumber\\
\leq&\ \sum_{k=1}^{n_\epsilon}\sum\limits_{\rvi\in\Delta_{k}}
       \Bigg(|a_{\rvi}|
  	   \bigg|\int_{\gX}\prod\limits_{j=1}^n(x_j)^{i_j}d\mu_S \nonumber\\
    &\ - \int_{\gX}\prod\limits_{j=1}^n(x_j)^{i_j}d\mu_T\bigg|\Bigg) 
       + \frac{\epsilon}{2} \nonumber\\
\leq&\ \sum_{k=1}^{n_\epsilon}\Bigg(a_{\Delta_k}
	   \sum\limits_{\rvi\in\Delta_k}
  	   \bigg|\int_{\gX}\prod\limits_{j=1}^n(x_j)^{i_j}d\mu_S \nonumber\\
    &\ - \int_{\gX}
       \prod\limits_{j=1}^n(x_j)^{i_j}d\mu_T\bigg|\Bigg) 
       + \frac{\epsilon}{2} \nonumber\\
  = &\ \sum_{k=1}^{n_\epsilon}a_{\Delta_k}d_{CM^{k}}\big(\gD_S, \gD_T\big)
  	   + \frac{\epsilon}{2} \nonumber\\
\leq&\ \frac{1}{2}a_{n_\epsilon}\sum_{k=1}^{n_\epsilon}
       d_{CM^{k}}\big(\gD_S,\gD_T\big) + \frac{\epsilon}{2},
\end{align}
where $a_{\Delta_k}=\max\limits_{\rvi\in\Delta_k}|a_{\rvi}|$ and $a_{n_\epsilon}=2\max\limits_{1\leq k\leq n_\epsilon}a_{\Delta_k}.$
Combining Equation~\ref{eqn:lem12a},~\ref{eqn:lem12b},~\ref{eqn:lem12c}, we prove the lemma.
\end{proof}
Note that the constants $a_{\Delta_k}$ can actually be meaningfully bounded when applying the Weierstrass Approximation Theorem. According to~\cite{Muradyan1977}, a sequence of positive numbers $\{M_k\}$ can be constructed, such that, for any $\epsilon>0$, there exists a polynomial $P_{n_\epsilon}$ such that $\left|P_{n_\epsilon}(\vx)-f(\vx)\right|<\epsilon$ and $\left|a_{\Delta_k}\right|<\epsilon M_k, \forall k=1,\ldots,n_{\epsilon}$.

\begin{lemma}[Lemma~6,~\cite{ben2010theory}]
\label{lem:ben2010}
For each $\gD_j\in\{\gD_1,\ldots,\gD_N\}$, let $S_j$ be a labeled sample set of size $\beta_jm$ drawn from $\mu_j$ and labeled by the groundtruth labeling function $f_j$. For any fixed weight vector $\valpha$, let $\hat{\epsilon}_{\valpha}(h)$ be the empirical $\valpha$-weighted error of some fixed hypothesis $h$ on these sample sets, and let $\epsilon_{\valpha}(h)$ be the true $\valpha$-weighted error, then
\[
\rmP\big[|\hat{\epsilon}_{\valpha}(h)-\epsilon_{\alpha}(h)|
  \geq\epsilon\big] 
\leq 2\exp
  \Bigg(\frac{-2m\epsilon^2}{\sum_{j=1}^N\frac{\alpha^2_j}{\beta_j}}\Bigg).
\] 
\label{lem:sample_ineq}
\end{lemma}
Lemma~\ref{lem:ben2010} is a slight modification
of the Hoeffding’s inequality for the empirical $\valpha$-weighted source error, which will be useful in proving the uniform convergence bound for hypothesis space of finite VC dimension.
Now we are ready to prove Theorem~\ref{thm:pairwise_bound}.
\pairwisebound*

\begin{proof}
Let $h^*_j=\argmin\limits_{h\in\gH}\{\epsilon_T(h)+\epsilon_j(h)\}$. Then for any $\epsilon>0$, there exists $N$ integers $\{n^j_\epsilon\}_{j=1}^N$ and $N$ constant $\{a_{n^j_\epsilon}\}_{j=1}^N$, such that
\begin{align}
\label{eq:intermediate}
&\ |\epsilon_{\valpha}(h) - \epsilon_T(h)| \nonumber\\
\leq&\ \bigg|\sum_{j=1}^N\alpha_j\epsilon_j(h) - \epsilon_T(h)\bigg|
	   \leq \sum_{j=1}^N\alpha_j|\epsilon_j(h) - \epsilon_T(h)|\nonumber\\
\leq&\ \sum_{j=1}^N\alpha_j\bigg( |\epsilon_j(h) - \epsilon_j(h,h^*_j)| 
  	   + |\epsilon_j(h,h^*_j) - \epsilon_T(h,h^*_j)| \nonumber\\
  	   &\ + |\epsilon_T(h,h^*_j) - \epsilon_T(h)| \bigg)\nonumber\\
\leq&\ \sum_{j=1}^N\alpha_j\big(\epsilon_j(h^*_j)
	   + |\epsilon_j(h,h^*_j) - \epsilon_T(h,h^*_j)| 
	   + \epsilon_T(h^*_j)\big) \nonumber\\
\leq&\ \sum_{j=1}^N\alpha_j\bigg(\lambda_j 
	   + \frac{1}{2}a^j_{n_\epsilon}\sum_{k=1}^{n^j_{\epsilon}}
	     d_{CM^{k}}\big(\gD_j, \gD_T\big) \bigg) + \frac{\epsilon}{2}.
\end{align}
The third inequality follows from the triangle inequality of classification error\footnote{For any labeling function $f_1,f_2,f_3$, we have $\epsilon(f_1,f_2)\leq\epsilon(f_1,f_3)+\epsilon(f_2,f_3)$.}~\cite{ben2007analysis,crammer2008learning}. The last inequality follows from the definition of $\lambda_j$ and Lemma~\ref{lem:keymoment}. Now using both~\Eqref{eq:intermediate} and Lemma~\ref{lem:ben2010}, we have for any $\delta\in(0,1)$ and any $\epsilon>0$, with probability $1-\delta$,
\begin{align*}
\epsilon_T(\hat{h})
\leq&\ \epsilon_{\valpha}(\hat{h}) + \frac{\epsilon}{2} \\
	   & + \sum\limits_{j=1}^N\alpha_j\bigg(\lambda_j 
	       + \frac{1}{2}a^j_{n_\epsilon}\sum_{k=1}^{n^j_{\epsilon}}
	         d_{CM^{k}}\big(\gD_j, \gD_T\big) \bigg) \\
\leq&\ \hat{\epsilon}_{\valpha}(\hat{h}) 
		   + \frac{1}{2}\eta_{\valpha,\vbeta,m,\delta}
		   + \frac{\epsilon}{2}\\
	   & + \sum\limits_{j=1}^N\alpha_j\bigg(\lambda_j
	       + \frac{1}{2}a^j_{n_\epsilon}\sum_{k=1}^{n^j_{\epsilon}}
	         d_{CM^{k}}\big(\gD_j, \gD_T\big) \bigg) \\
\leq&\ \hat{\epsilon}_{\valpha}(h^*_T) 
		   + \frac{1}{2}\eta_{\valpha,\vbeta,m,\delta}
		   + \frac{\epsilon}{2}\\
	   & + \sum\limits_{j=1}^N\alpha_j\bigg(\lambda_j
	       + \frac{1}{2}a^j_{n_\epsilon}\sum_{k=1}^{n^j_{\epsilon}}
	         d_{CM^{k}}\big(\gD_j, \gD_T\big) \bigg) \\	
\leq&\ \epsilon_{\valpha}(h^*_T) 
		   + \eta_{\valpha,\vbeta,m,\delta} + \frac{\epsilon}{2}\\
	   & + \sum\limits_{j=1}^N\alpha_j\bigg(\lambda_j
	       + \frac{1}{2}a^j_{n_\epsilon}\sum_{k=1}^{n^j_{\epsilon}}
	         d_{CM^{k}}\big(\gD_j, \gD_T\big) \bigg) \\
\leq&\ \epsilon_T(h^*_T) 
		   + \eta_{\valpha,\vbeta,m,\delta} + \epsilon\\
	   & + \sum\limits_{j=1}^N\alpha_j\bigg(2\lambda_j
	       + a^j_{n_\epsilon}\sum_{k=1}^{n^j_{\epsilon}}
	         d_{CM^{k}}\big(\gD_j, \gD_T\big) \bigg). 
\end{align*}
The first and the last inequalities follow from~\Eqref{eq:intermediate}, the second and the fourth inequalities follow from applying Lemma~\ref{lem:ben2010} (instead of standard Hoeffding's inequality) in the standard proof of uniform convergence
for empirical risk minimizers~\cite{vapnik2015uniform}. The third inequality follows from the definition of $\hat{h}$.
\end{proof}

To better understand the bounds in Theorem~\ref{thm:pairwise_bound}, the second term of the bound is the VC-dimension based generalization error, which is the upper bound of the difference between the empirical error $\hat{\epsilon}_{\valpha}$ and the true expected error $\epsilon_{\valpha}$. The last term (a summation), as shown in Equation~\ref{eq:intermediate}, characterizes the upper bound of the difference between the $\valpha$-weighted error $\epsilon_{\valpha}$ and the target error $\epsilon_T$. The constants $\{a_{n^j_\epsilon}\}_{j=1}^N$ in this term can be meaningfully bounded, as explained at the end of the proof of Lemma~\ref{lem:keymoment}.

Note that the bound explicitly depends on cross-moment divergence terms $d_{CM^{k}}\big(\gD_j, \gD_T\big)$, and thus sheds new light on the theoretical motivation of moment matching approaches, including our proposed approach and many existing approaches for both single and multiple source domain adaptation.
To the best of our knowledge, this is the first target error bound in the literature of domain adaptation that explicitly incorporates a moment-based divergence between the source(s) and the target domains.

%% file: Appendix/4_digit_exp_conf.tex
\subsection{Details of Digit Experiments}
\label{sec_digit_exp}

\noindent \textbf{Network Architecture} In our digital experiments (Table~\ref{tab_digit_five}), our feature extractor is composed of three \textit{conv} layers and two \textit{fc} layers. We present the \textit{conv} layers as (\textit{input}, \textit{output}, \textit{kernel}, \textit{stride}, \textit{padding}) and \textit{fc} layers as (\textit{input}, \textit{output}). The three \textit{conv} layers are: \textit{conv1} (3, 64, 5, 1, 2), \textit{conv2} (64, 64, 5, 1, 2), \textit{conv3} (64, 128, 5, 1, 2). The two \textit{fc} layers are: \textit{fc1} (8192, 3072), \textit{fc2} (3072, 2048). The architecture of the feature generator is: (\textit{conv1}, \textit{bn1}, \textit{relu1}, \textit{pool1})-(\textit{conv2}, \textit{bn2}, \textit{relu2}, \textit{pool2})-(\textit{conv3}, \textit{bn3}, \textit{relu3})-(\textit{fc1}, \textit{bn4}, \textit{relu4}, \textit{dropout})-(\textit{fc2}, \textit{bn5}, \textit{relu5}). The classifier is a single \textit{fc} layer, \ie \textit{fc3} (2048, 10).

\begin{figure}[t]
    \centering
    \includegraphics{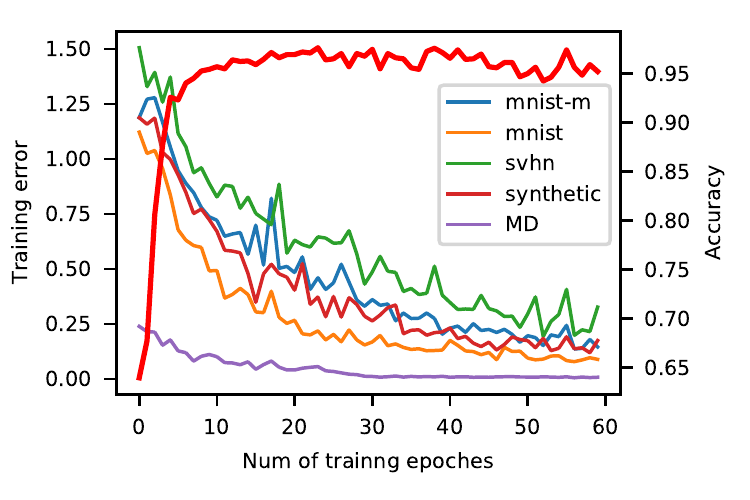}
    \vspace{-0.2cm}
    \caption{Analysis: training error of each classifier, Moment Discrepancy (MD, defined in Equation~\ref{eqn:MD2}), and accuracy (red bold line) \textit{w.r.t.} training epochs in {\BV{mm,mt,sv,sy}}$\rightarrow${\BV{up}} setting. The bold red line shows how the accuracy changes \textit{w.r.t.} training epochs.}
    \label{fig_covergence}
    \vspace{-0.2cm}
\end{figure}

\noindent \textbf{Convergence Analysis}
As our framework involves multiple classifiers and the model is trained with multiple losses, we visualize the learning procedure of {\BV{mm,mt,sv,sy}}$\rightarrow${\BV{up}} setting in Figure~\ref{fig_covergence}. The figure shows the training errors of multiple classifiers are decreasing, despite some frequent deviations. The MD (Moment Discrepancy) loss decreases in a steady way, demonstrating that the MD loss and the cross-entropy loss gradually converge.

%% file: Appendix/3_table_main_result_full.tex

%% file: Appendix/6_resnet_table.tex
\subsection{ResNet baselines}
\label{sec_supp_resnet_baselinese}
We report ResNet~\cite{he2015deep} source only baselines in Table~\ref{tab_resnet_source_only}. The MCD~\cite{MCD} and the SE~\cite{SE} methods (In Table~\ref{tab_challengeI}) are based on ResNet101 and ResNet152, respectively. We have observed these two methods both perform worse than their source only baselines, which indicates negative transfer~\cite{pan2010survey} phenomenon occurs in these two scenarios. Exploring why negative transfer happens is beyond the scope of this literature. One preliminary guess is due to the large number of categories.

\subsection{Train/Test Split}
\label{sec_supp_train_test_split}
\vspace{-0.2cm}
We show the detailed number of images we used in our experiments in Table~\ref{tab_split_detail}. For each domain, we follow a 70\%/30\% schema to split the dataset to training and testing trunk. The ``Per-Class'' row shows the average number of images that each category contains.

\begin{table}[t]
\setlength{\tabcolsep}{0.22em}
\renewcommand{\arraystretch}{0.8}
\footnotesize{
    \centering
    \begin{tabular}{c|c| c| c| c |c |c |c}
     &\BV{clp} & \BV{inf} & \BV{pnt} & \BV{qdr} & \BV{rel} & \BV{skt} & Total\\
     \hline
    Train & 34,019 &	37,087&	52,867	&120,750&	122,563&	49,115	&416,401 \\  
    Test  &  14,818	&16,114	&22,892&	51,750&	52,764&	21,271&	179,609 \\ 
    \hline
    Total & 48,837&	53,201&	75,759&	172,500&	175,327&	70,386&	596,010 \\
    Per-Class & 141&	154&	219&	500&	508&	204& 1,728 \\
    \end{tabular}
    }
    \caption{\textbf{Train/Test split.} We split \DATASETNAME with a 70\%/30\% ratio. The ``Per-Class" row shows the average number of images that each category contains.}

    \label{tab_split_detail}
\end{table}

%% file: Appendix/y_6_figures.tex
\begin{figure*}[t]
    \centering
   \includegraphics[width=\linewidth]{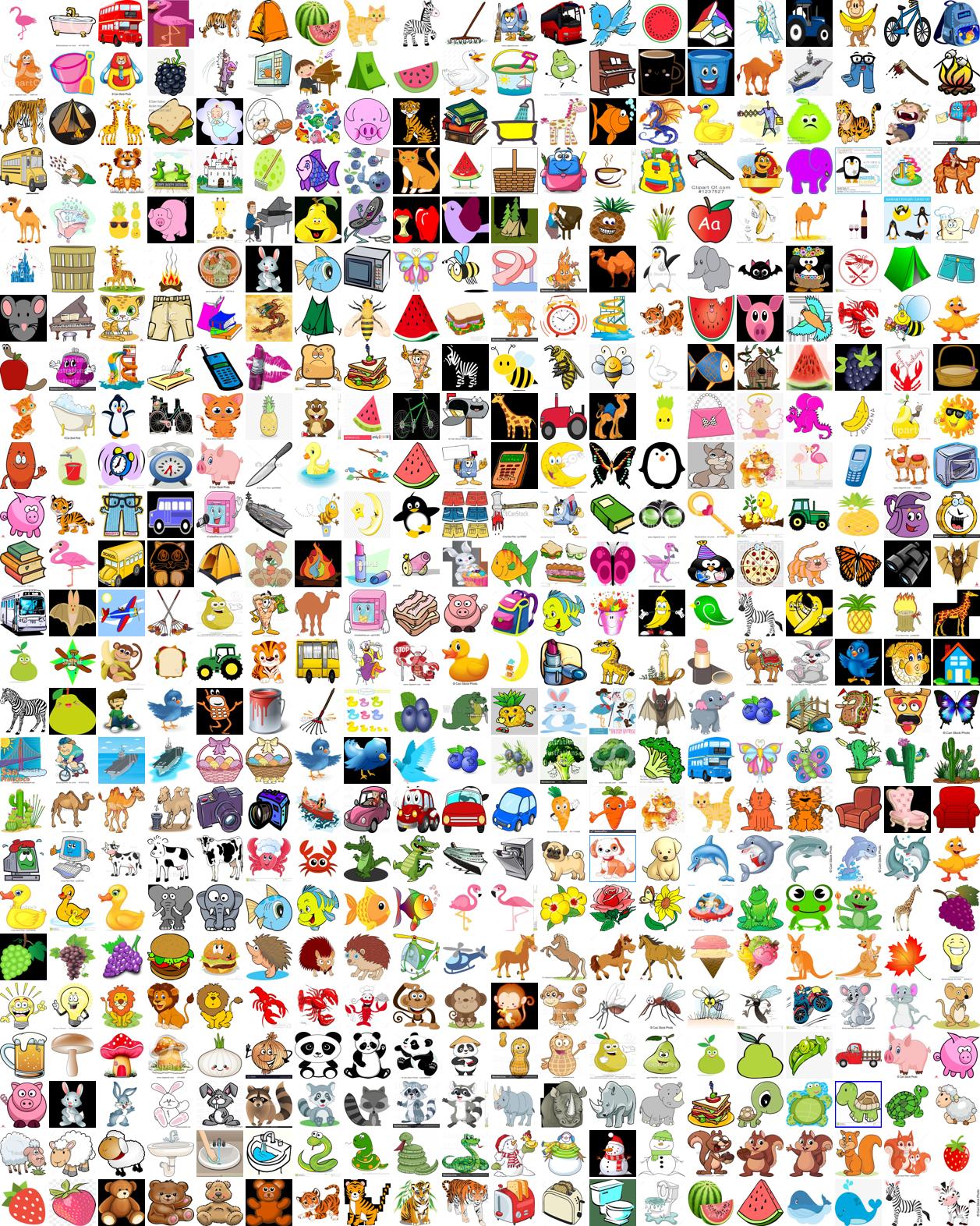}
    \caption{Images sampled from \textit{clipart} domain of the \DATASETNAME dataset.}
    \label{fig_clipart}
\end{figure*}

\begin{figure*}[t]
    \centering
   \includegraphics[width=\linewidth]{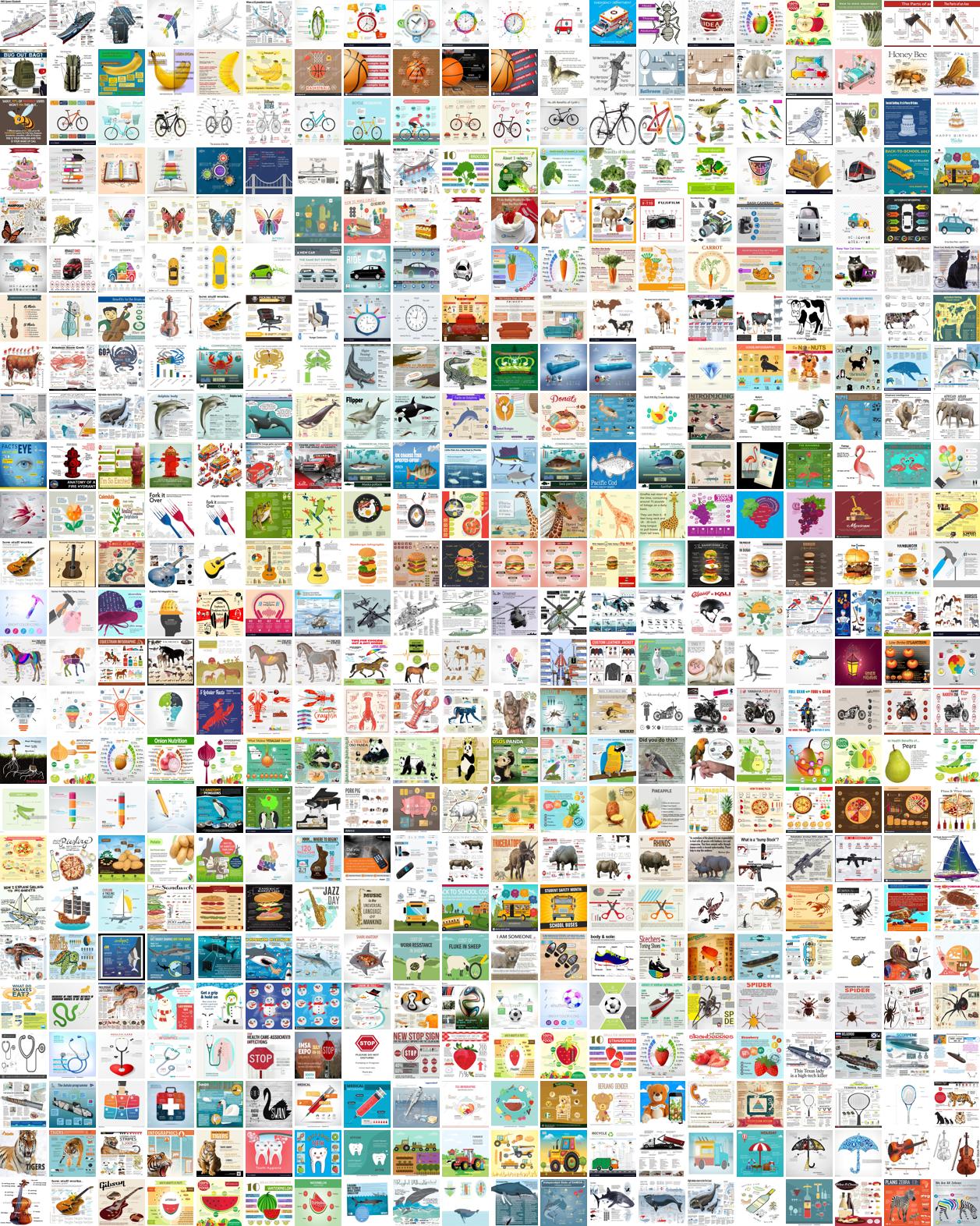}
    \caption{Images sampled from \textit{infograph} domain of the \DATASETNAME dataset.}
    \label{fig_infograph}
\end{figure*}

\begin{figure*}[t]
    \centering
   \includegraphics[width=\linewidth]{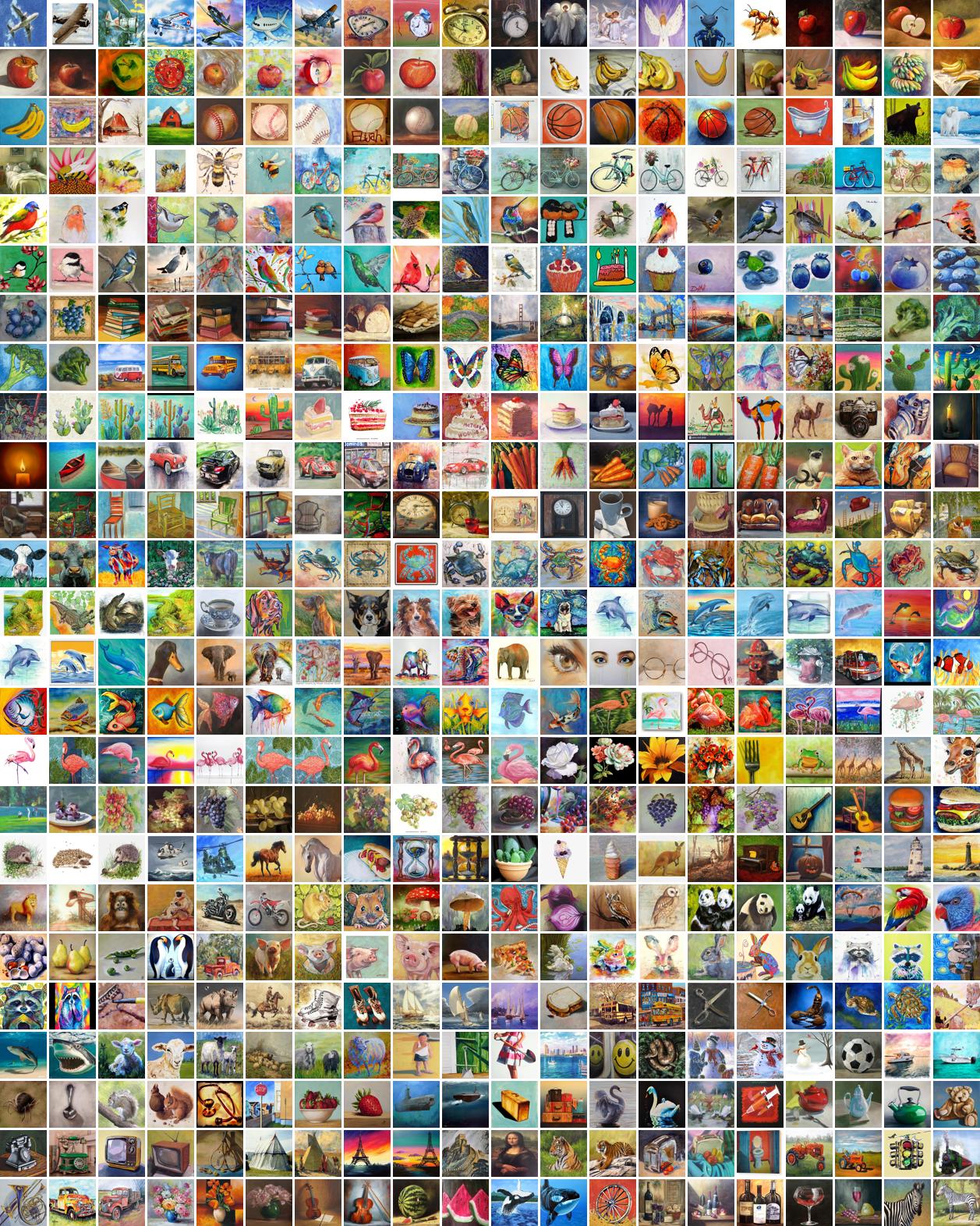}
    \caption{Images sampled from \textit{painting} domain of the \DATASETNAME dataset.}
    \label{fig_painting}
\end{figure*}

\begin{figure*}[t]
    \centering
   \includegraphics[width=\linewidth]{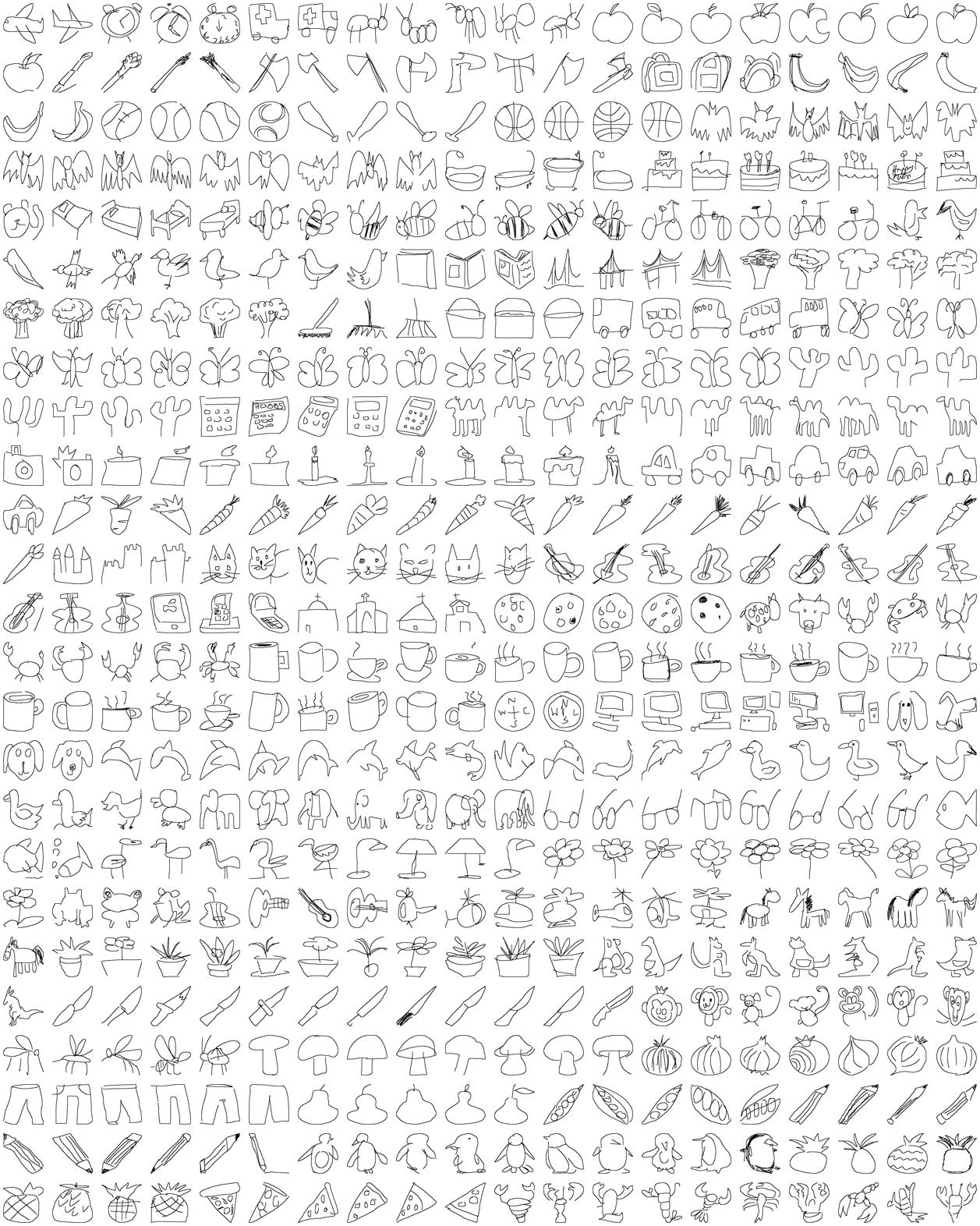}
    \caption{Images sampled from \textit{quickdraw} domain of the \DATASETNAME dataset.}
    \label{fig_quickdraw}
\end{figure*}

\begin{figure*}[t]
    \centering
   \includegraphics[width=\linewidth]{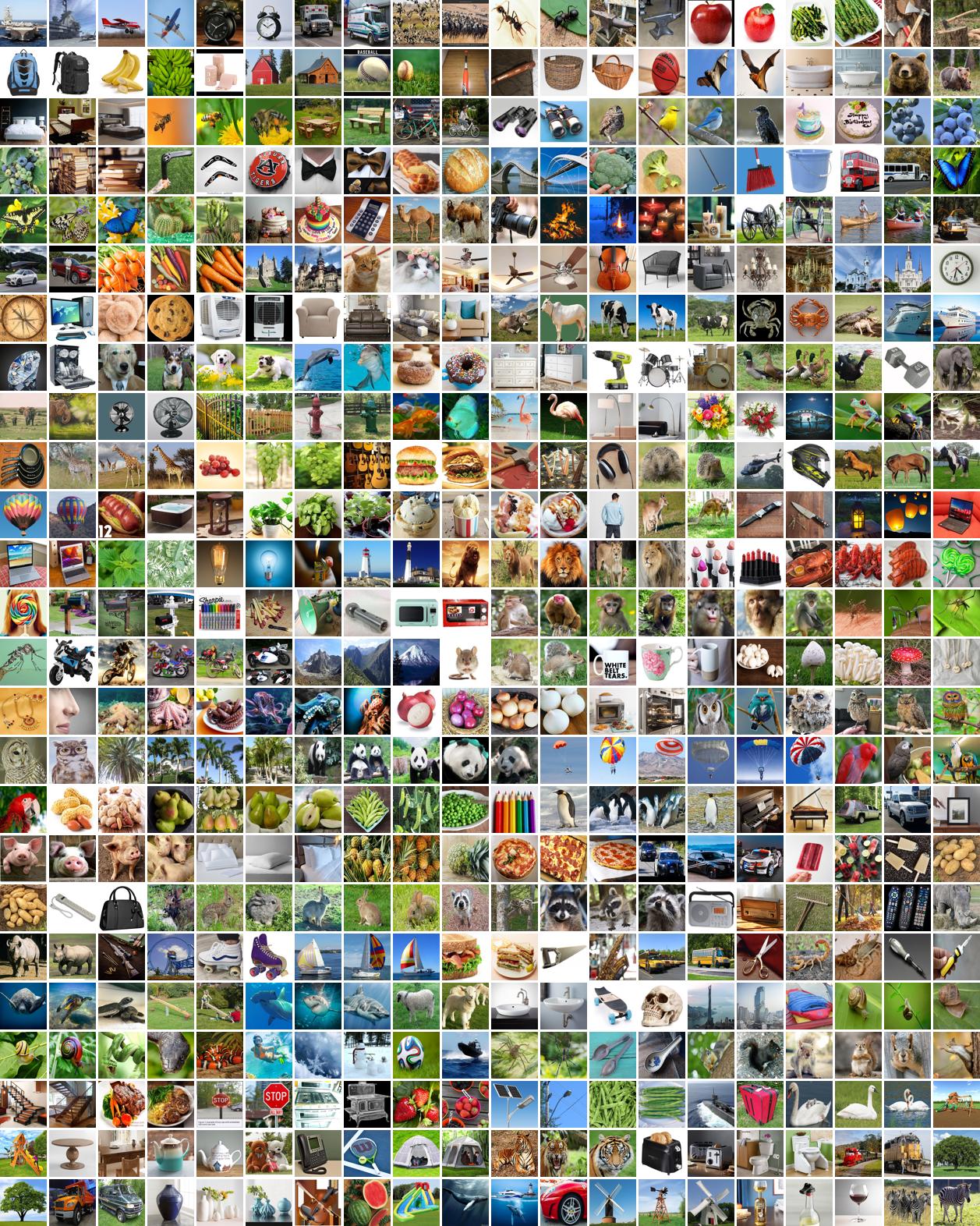}
    \caption{Images sampled from \textit{real} domain of the \DATASETNAME dataset.}
    \label{fig_real}
\end{figure*}

\begin{figure*}[t]
    \centering
   \includegraphics[width=\linewidth]{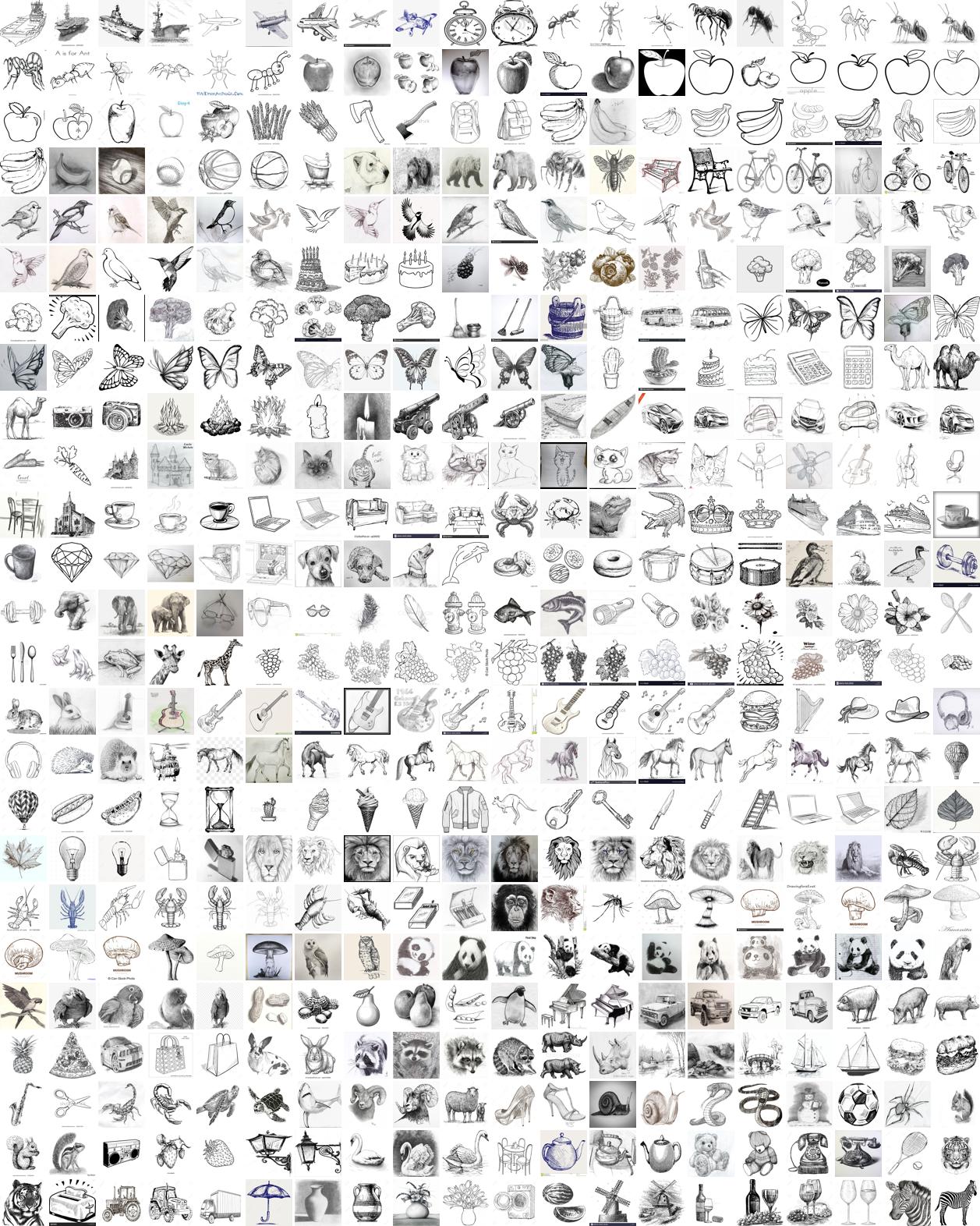}
    \caption{Images sampled from \textit{sketch} domain of the \DATASETNAME dataset.}
    \label{fig_sketch}
\end{figure*}

%% file: Appendix/z_table_11_12_13.tex
\clearpage
\begin{table*}
\centering
\renewcommand{\arraystretch}{1.2}
\renewcommand{\tabcolsep}{1.2mm}
\resizebox{\linewidth}{!}{
\begin{tabular}{|l l l l l l l l |l l l l l l l l|l l l l l l l l|}
\hline
\cline{1-24}
              \multicolumn{24}{|c|}{\BV{Furniture}}\\
\cline{1-24}
\hline
\textbf{class} & \textbf{clp} & \textbf{inf} & \textbf{pnt} & \textbf{qdr} & \textbf{rel} & \textbf{skt} & \textbf{total} &
\textbf{class} & \textbf{clp} & \textbf{inf} & \textbf{pnt} & \textbf{qdr} & \textbf{rel} & \textbf{skt} & \textbf{total} &
\textbf{class} & \textbf{clp} & \textbf{inf} & \textbf{pnt} & \textbf{qdr} & \textbf{rel} & \textbf{skt} & \textbf{total}\\
\hline
bathtub & 100 & 135 & 45 & 500 & 517 & 210 & 1507 &
bed & 197 & 180 & 46 & 500 & 724 & 188 & 1835 &
bench & 71 & 47 & 167 & 500 & 662 & 290 & 1737 \\
ceiling fan & 35 & 63 & 38 & 500 & 217 & 25 & 878 &
chair & 94 & 148 & 53 & 500 & 320 & 96 & 1211 &
chandelier & 223 & 29 & 57 & 500 & 393 & 34 & 1236 \\
couch & 232 & 61 & 26 & 500 & 601 & 60 & 1480 &
door & 81 & 49 & 347 & 500 & 371 & 361 & 1709 &
dresser & 41 & 23 & 141 & 500 & 234 & 13 & 952 \\
fence & 165 & 99 & 49 & 500 & 770 & 140 & 1723 &
fireplace & 138 & 98 & 15 & 500 & 700 & 123 & 1574 &
floor lamp & 180 & 100 & 10 & 500 & 246 & 278 & 1314 \\
hot tub & 144 & 86 & 197 & 500 & 757 & 49 & 1733 &
ladder & 74 & 96 & 418 & 500 & 442 & 244 & 1774 &
lantern & 179 & 58 & 218 & 500 & 526 & 40 & 1521 \\
mailbox & 18 & 45 & 101 & 500 & 595 & 151 & 1410 &
picture frame & 88 & 60 & 372 & 500 & 207 & 115 & 1342 &
pillow & 151 & 170 & 144 & 500 & 656 & 115 & 1736 \\
postcard & 91 & 37 & 88 & 500 & 636 & 49 & 1401 &
see saw & 299 & 28 & 166 & 500 & 273 & 519 & 1785 &
sink & 133 & 32 & 94 & 500 & 231 & 464 & 1454 \\
sleeping bag & 96 & 14 & 17 & 500 & 406 & 591 & 1624 &
stairs & 386 & 282 & 27 & 500 & 353 & 525 & 2073 &
stove & 256 & 255 & 16 & 500 & 614 & 269 & 1910 \\
streetlight & 326 & 113 & 537 & 500 & 463 & 268 & 2207 &
suitcase & 377 & 224 & 187 & 500 & 432 & 309 & 2029 &
swing set & 143 & 35 & 129 & 500 & 556 & 96 & 1459 \\
table & 297 & 736 & 104 & 500 & 563 & 300 & 2500 &
teapot & 222 & 209 & 391 & 500 & 631 & 327 & 2280 &
toilet & 175 & 519 & 31 & 500 & 583 & 118 & 1926 \\
toothbrush & 159 & 556 & 11 & 500 & 582 & 235 & 2043 &
toothpaste & 105 & 468 & 31 & 500 & 511 & 198 & 1813 &
umbrella & 145 & 511 & 299 & 500 & 362 & 297 & 2114 \\
vase & 161 & 319 & 262 & 500 & 632 & 187 & 2061 &
wine glass & 220 & 628 & 168 & 500 & 338 & 245 & 2099 &
& & & & & & & \\
\hline
\cline{1-24}
              \multicolumn{24}{|c|}{\BV{Mammal}}\\
\cline{1-24}
\hline
bat & 35 & 99 & 306 & 500 & 361 & 160 & 1461 &
bear & 124 & 81 & 379 & 500 & 585 & 178 & 1847 &
camel & 154 & 31 & 289 & 500 & 493 & 130 & 1597 \\
cat & 43 & 172 & 344 & 500 & 796 & 130 & 1985 &
cow & 188 & 134 & 156 & 500 & 541 & 17 & 1536 &
dog & 70 & 225 & 721 & 500 & 782 & 311 & 2609 \\
dolphin & 84 & 165 & 401 & 500 & 581 & 85 & 1816 &
elephant & 115 & 188 & 425 & 500 & 789 & 266 & 2283 &
giraffe & 134 & 172 & 105 & 500 & 594 & 186 & 1691 \\
hedgehog & 138 & 48 & 248 & 500 & 727 & 109 & 1770 &
horse & 201 & 216 & 521 & 500 & 645 & 103 & 2186 &
kangaroo & 106 & 60 & 214 & 500 & 613 & 122 & 1615 \\
lion & 46 & 64 & 505 & 500 & 516 & 330 & 1961 &
monkey & 123 & 85 & 405 & 500 & 699 & 166 & 1978 &
mouse & 74 & 50 & 445 & 500 & 147 & 127 & 1343 \\
panda & 87 & 86 & 264 & 500 & 587 & 79 & 1603 &
pig & 93 & 203 & 326 & 500 & 577 & 227 & 1926 &
rabbit & 105 & 135 & 269 & 500 & 695 & 94 & 1798 \\
raccoon & 187 & 24 & 249 & 500 & 676 & 348 & 1984 &
rhinoceros & 102 & 91 & 220 & 500 & 684 & 183 & 1780 &
sheep & 114 & 70 & 334 & 500 & 796 & 475 & 2289 \\
squirrel & 221 & 180 & 779 & 500 & 693 & 389 & 2762 &
tiger & 315 & 285 & 422 & 500 & 607 & 386 & 2515 &
whale & 343 & 432 & 357 & 500 & 671 & 272 & 2575 \\
zebra & 235 & 306 & 298 & 500 & 683 & 278 & 2300 &
& & & & & & & &
& & & & & & & \\
\hline
\cline{1-24}
              \multicolumn{24}{|c|}{\BV{Tool}}\\
\cline{1-24}
\hline
anvil & 122 & 23 & 152 & 500 & 332 & 91 & 1220 &
axe & 48 & 92 & 219 & 500 & 382 & 219 & 1460 &
bandage & 47 & 322 & 197 & 500 & 399 & 56 & 1521 \\
basket & 78 & 219 & 417 & 500 & 444 & 192 & 1850 &
boomerang & 92 & 45 & 41 & 500 & 628 & 120 & 1426 &
bottlecap & 118 & 26 & 538 & 500 & 606 & 139 & 1927 \\
broom & 171 & 35 & 84 & 500 & 639 & 234 & 1663 &
bucket & 142 & 56 & 61 & 500 & 335 & 162 & 1256 &
compass & 191 & 36 & 78 & 500 & 272 & 14 & 1091 \\
drill & 136 & 44 & 21 & 500 & 573 & 144 & 1418 &
dumbbell & 387 & 86 & 189 & 500 & 581 & 190 & 1933 &
hammer & 147 & 70 & 46 & 500 & 347 & 71 & 1181 \\
key & 59 & 68 & 97 & 500 & 229 & 137 & 1090 &
nail & 41 & 256 & 838 & 500 & 674 & 23 & 2332 &
paint can & 60 & 42 & 172 & 500 & 560 & 34 & 1368 \\
passport & 26 & 120 & 34 & 500 & 535 & 97 & 1312 &
pliers & 38 & 65 & 293 & 500 & 453 & 163 & 1512 &
rake & 119 & 66 & 58 & 500 & 594 & 93 & 1430 \\
rifle & 83 & 149 & 240 & 500 & 520 & 122 & 1614 &
saw & 76 & 34 & 150 & 500 & 118 & 110 & 988 &
screwdriver & 205 & 34 & 73 & 500 & 417 & 373 & 1602 \\
shovel & 214 & 17 & 112 & 500 & 450 & 630 & 1923 &
skateboard & 263 & 50 & 152 & 500 & 557 & 419 & 1941 &
stethoscope & 343 & 107 & 346 & 500 & 496 & 237 & 2029 \\
stitches & 206 & 285 & 17 & 500 & 207 & 34 & 1249 &
sword & 139 & 124 & 470 & 500 & 591 & 384 & 2208 &
syringe & 128 & 240 & 10 & 500 & 589 & 222 & 1689 \\
wheel & 133 & 385 & 19 & 500 & 410 & 166 & 1613 &
& & & & & & & &
& & & & & & & \\
\hline
\cline{1-24}
              \multicolumn{24}{|c|}{\BV{Cloth}}\\
\cline{1-24}
\hline

belt & 137 & 95 & 17 & 500 & 661 & 125 & 1535 &
bowtie & 146 & 95 & 292 & 500 & 533 & 327 & 1893 &
bracelet & 293 & 123 & 150 & 500 & 715 & 300 & 2081 \\
camouflage & 181 & 27 & 72 & 500 & 124 & 69 & 973 &
crown & 208 & 17 & 81 & 500 & 170 & 176 & 1152 &
diamond & 207 & 109 & 17 & 500 & 577 & 117 & 1527 \\
eyeglasses & 201 & 118 & 83 & 500 & 680 & 219 & 1801 &
flip flops & 147 & 53 & 206 & 500 & 525 & 120 & 1551 &
hat & 120 & 201 & 400 & 500 & 529 & 77 & 1827 \\
helmet & 163 & 263 & 27 & 500 & 622 & 210 & 1785 &
jacket & 72 & 82 & 272 & 500 & 457 & 84 & 1467 &
lipstick & 101 & 104 & 196 & 500 & 446 & 110 & 1457 \\
necklace & 83 & 115 & 347 & 500 & 697 & 114 & 1856 &
pants & 16 & 173 & 381 & 500 & 398 & 136 & 1604 &
purse & 41 & 119 & 49 & 500 & 544 & 228 & 1481 \\
rollerskates & 204 & 49 & 322 & 500 & 493 & 141 & 1709 &
shoe & 127 & 291 & 260 & 500 & 587 & 645 & 2410 &
shorts & 140 & 29 & 161 & 500 & 443 & 529 & 1802 \\
sock & 167 & 453 & 31 & 500 & 531 & 453 & 2135 &
sweater & 222 & 92 & 153 & 500 & 579 & 167 & 1713 &
t-shirt & 98 & 320 & 12 & 500 & 367 & 155 & 1452 \\
underwear & 253 & 354 & 12 & 500 & 286 & 132 & 1537 &
wristwatch & 285 & 470 & 18 & 500 & 553 & 224 & 2050 &
& & & & & & & \\
\hline
\cline{1-24}
              \multicolumn{24}{|c|}{\BV{Electricity}}\\
\cline{1-24}
\hline
calculator & 55 & 28 & 12 & 500 & 374 & 69 & 1038 &
camera & 58 & 66 & 156 & 500 & 480 & 109 & 1369 &
cell phone & 38 & 170 & 136 & 500 & 520 & 23 & 1387 \\
computer & 287 & 97 & 19 & 500 & 362 & 31 & 1296 &
cooler & 214 & 21 & 13 & 500 & 528 & 90 & 1366 &
dishwasher & 109 & 47 & 107 & 500 & 508 & 40 & 1311 \\
fan & 148 & 49 & 16 & 500 & 460 & 66 & 1239 &
flashlight & 221 & 62 & 418 & 500 & 461 & 95 & 1757 &
headphones & 285 & 224 & 181 & 500 & 551 & 188 & 1929 \\
keyboard & 32 & 95 & 370 & 500 & 503 & 64 & 1564 &
laptop & 26 & 118 & 161 & 500 & 387 & 319 & 1511 &
light bulb & 12 & 185 & 482 & 500 & 262 & 405 & 1846 \\
megaphone & 73 & 91 & 160 & 500 & 560 & 189 & 1573 &
microphone & 143 & 70 & 152 & 500 & 562 & 156 & 1583 &
microwave & 16 & 114 & 10 & 500 & 338 & 170 & 1148 \\
oven & 14 & 59 & 11 & 500 & 492 & 176 & 1252 &
power outlet & 25 & 76 & 102 & 500 & 620 & 95 & 1418 &
radio & 30 & 101 & 65 & 500 & 398 & 165 & 1259 \\
remote control & 117 & 70 & 111 & 500 & 554 & 47 & 1399 &
spreadsheet & 187 & 397 & 34 & 500 & 751 & 677 & 2546 &
stereo & 289 & 334 & 12 & 500 & 211 & 90 & 1436 \\
telephone & 148 & 279 & 78 & 500 & 479 & 255 & 1739 &
television & 136 & 546 & 51 & 500 & 400 & 127 & 1760 &
toaster & 196 & 337 & 107 & 500 & 536 & 267 & 1943 \\
washing machine & 265 & 519 & 15 & 500 & 466 & 155 & 1920 &
& & & & & & & &
& & & & & & & \\
\hline
\end{tabular}

}
\caption{Detailed statistics of the \DATASETNAME dataset.}
\label{dataset_1}
\end{table*}

\clearpage
\begin{table*}
\centering
\renewcommand{\arraystretch}{1.2}
\renewcommand{\tabcolsep}{1.2mm}
\resizebox{\linewidth}{!}{
\begin{tabular}{|l l l l l l l l |l l l l l l l l|l l l l l l l l|}
\hline
\cline{1-24}
              \multicolumn{24}{|c|}{\BV{Building}}\\
\cline{1-24}
\hline
\textbf{class} & \textbf{clp} & \textbf{inf} & \textbf{pnt} & \textbf{qdr} & \textbf{rel} & \textbf{skt} & \textbf{total} &
\textbf{class} & \textbf{clp} & \textbf{inf} & \textbf{pnt} & \textbf{qdr} & \textbf{rel} & \textbf{skt} & \textbf{total} &
\textbf{class} & \textbf{clp} & \textbf{inf} & \textbf{pnt} & \textbf{qdr} & \textbf{rel} & \textbf{skt} & \textbf{total}\\
\hline

The Eiffel Tower & 114 & 190 & 321 & 500 & 553 & 276 & 1954 &
The Great Wall & 116 & 80 & 159 & 500 & 530 & 148 & 1533 &
barn & 157 & 150 & 426 & 500 & 313 & 201 & 1747 \\
bridge & 66 & 61 & 471 & 500 & 769 & 335 & 2202 &
castle & 47 & 123 & 225 & 500 & 682 & 56 & 1633 &
church & 54 & 20 & 142 & 500 & 668 & 35 & 1419 \\
diving board & 182 & 12 & 127 & 500 & 593 & 71 & 1485 &
garden & 63 & 291 & 213 & 500 & 815 & 98 & 1980 &
garden hose & 147 & 48 & 179 & 500 & 534 & 84 & 1492 \\
golf club & 207 & 169 & 650 & 500 & 552 & 695 & 2773 &
hospital & 95 & 48 & 50 & 500 & 674 & 24 & 1391 &
house & 108 & 306 & 105 & 500 & 374 & 144 & 1537 \\
jail & 104 & 26 & 54 & 500 & 587 & 94 & 1365 &
lighthouse & 123 & 66 & 411 & 500 & 722 & 384 & 2206 &
pond & 105 & 72 & 603 & 500 & 777 & 95 & 2152 \\
pool & 139 & 173 & 90 & 500 & 680 & 103 & 1685 &
skyscraper & 195 & 159 & 179 & 500 & 284 & 466 & 1783 &
square & 163 & 211 & 144 & 500 & 98 & 727 & 1843 \\
tent & 153 & 234 & 141 & 500 & 590 & 339 & 1957 &
waterslide & 159 & 328 & 12 & 500 & 606 & 115 & 1720 &
windmill & 245 & 372 & 397 & 500 & 635 & 245 & 2394 \\

\hline
\cline{1-24}
              \multicolumn{24}{|c|}{\BV{Office}}\\
\cline{1-24}

alarm clock & 93 & 23 & 84 & 500 & 521 & 202 & 1423 &
backpack & 33 & 341 & 265 & 500 & 439 & 220 & 1798 &
bandage & 47 & 322 & 197 & 500 & 399 & 56 & 1521 \\
binoculars & 246 & 55 & 148 & 500 & 402 & 266 & 1617 &
book & 167 & 188 & 65 & 500 & 731 & 146 & 1797 &
calendar & 66 & 54 & 44 & 500 & 176 & 59 & 899 \\
candle & 151 & 68 & 261 & 500 & 621 & 77 & 1678 &
clock & 69 & 50 & 266 & 500 & 619 & 44 & 1548 &
coffee cup & 357 & 191 & 185 & 500 & 643 & 33 & 1909 \\
crayon & 141 & 41 & 19 & 500 & 512 & 285 & 1498 &
cup & 128 & 52 & 582 & 500 & 406 & 396 & 2064 &
envelope & 147 & 60 & 291 & 500 & 665 & 183 & 1846 \\
eraser & 138 & 34 & 17 & 500 & 356 & 51 & 1096 &
map & 42 & 206 & 423 & 500 & 507 & 193 & 1871 &
marker & 57 & 44 & 103 & 500 & 336 & 240 & 1280 \\
mug & 168 & 41 & 500 & 500 & 598 & 186 & 1993 &
nail & 41 & 256 & 838 & 500 & 674 & 23 & 2332 &
paintbrush & 25 & 28 & 365 & 500 & 413 & 75 & 1406 \\
paper clip & 122 & 25 & 112 & 500 & 549 & 119 & 1427 &
pencil & 51 & 369 & 183 & 500 & 461 & 26 & 1590 &
scissors & 205 & 103 & 65 & 500 & 568 & 437 & 1878 \\

\cline{1-24}
              \multicolumn{24}{|c|}{\BV{Human Body}}\\
\cline{1-24}

arm & 50 & 129 & 422 & 500 & 235 & 249 & 1585 &
beard & 156 & 93 & 373 & 500 & 728 & 481 & 2331 &
brain & 76 & 283 & 233 & 500 & 724 & 270 & 2086 \\
ear & 101 & 58 & 187 & 500 & 348 & 199 & 1393 &
elbow & 199 & 74 & 97 & 500 & 398 & 155 & 1423 &
eye & 108 & 168 & 292 & 500 & 695 & 489 & 2252 \\
face & 54 & 110 & 20 & 500 & 696 & 452 & 1832 &
finger & 153 & 71 & 57 & 500 & 625 & 283 & 1689 &
foot & 85 & 111 & 86 & 500 & 558 & 261 & 1601 \\
goatee & 255 & 236 & 129 & 500 & 562 & 219 & 1901 &
hand & 97 & 268 & 262 & 500 & 563 & 264 & 1954 &
knee & 45 & 56 & 130 & 500 & 505 & 273 & 1509 \\
leg & 89 & 174 & 178 & 500 & 659 & 145 & 1745 &
moustache & 222 & 28 & 430 & 500 & 424 & 107 & 1711 &
mouth & 110 & 103 & 51 & 500 & 457 & 172 & 1393 \\
nose & 57 & 226 & 512 & 500 & 362 & 103 & 1760 &
skull & 178 & 29 & 189 & 500 & 329 & 600 & 1825 &
smiley face & 113 & 46 & 77 & 500 & 226 & 441 & 1403 \\
toe & 85 & 407 & 12 & 500 & 356 & 78 & 1438 &
tooth & 101 & 473 & 109 & 500 & 257 & 181 & 1621 &
& & & & & & & \\

\cline{1-24}
              \multicolumn{24}{|c|}{\BV{Road Transportation}}\\
\cline{1-24}
ambulance & 80 & 20 & 74 & 500 & 623 & 115 & 1412 &
bicycle & 71 & 272 & 196 & 500 & 705 & 343 & 2087 &
bulldozer & 111 & 55 & 58 & 500 & 635 & 199 & 1558 \\
bus & 101 & 183 & 112 & 500 & 745 & 233 & 1874 &
car & 99 & 356 & 45 & 500 & 564 & 145 & 1709 &
firetruck & 167 & 39 & 359 & 500 & 562 & 328 & 1955 \\
motorbike & 42 & 209 & 106 & 500 & 772 & 209 & 1838 &
pickup truck & 46 & 116 & 143 & 500 & 619 & 188 & 1612 &
police car & 104 & 51 & 87 & 500 & 740 & 119 & 1601 \\
roller coaster & 143 & 46 & 75 & 500 & 637 & 61 & 1462 &
school bus & 230 & 142 & 66 & 500 & 478 & 405 & 1821 &
tractor & 154 & 316 & 183 & 500 & 636 & 263 & 2052 \\
train & 109 & 373 & 406 & 500 & 681 & 240 & 2309 &
truck & 117 & 678 & 158 & 500 & 673 & 265 & 2391 &
van & 207 & 806 & 12 & 500 & 442 & 138 & 2105 \\
\cline{1-24}
              \multicolumn{24}{|c|}{\BV{Food}}\\
\cline{1-24}
birthday cake & 165 & 69 & 119 & 500 & 307 & 233 & 1393 &
bread & 197 & 232 & 315 & 500 & 794 & 276 & 2314 &
cake & 108 & 140 & 172 & 500 & 786 & 77 & 1783 \\
cookie & 97 & 78 & 54 & 500 & 677 & 33 & 1439 &
donut & 139 & 65 & 373 & 500 & 630 & 127 & 1834 &
hamburger & 187 & 210 & 147 & 500 & 559 & 185 & 1788 \\
hot dog & 38 & 138 & 148 & 500 & 644 & 143 & 1611 &
ice cream & 160 & 187 & 311 & 500 & 657 & 184 & 1999 &
lollipop & 74 & 28 & 252 & 500 & 607 & 106 & 1567 \\
peanut & 84 & 60 & 82 & 500 & 423 & 130 & 1279 &
pizza & 77 & 157 & 127 & 500 & 600 & 202 & 1663 &
popsicle & 288 & 79 & 171 & 500 & 639 & 117 & 1794 \\
sandwich & 189 & 110 & 139 & 500 & 579 & 132 & 1649 &
steak & 155 & 360 & 50 & 500 & 758 & 238 & 2061 &
 & & & & & & & \\

\cline{1-24}
              \multicolumn{24}{|c|}{\BV{Nature}}\\
\cline{1-24}
beach & 105 & 183 & 499 & 500 & 622 & 79 & 1988 &
cloud & 172 & 142 & 278 & 500 & 324 & 100 & 1516 &
hurricane & 92 & 68 & 133 & 500 & 420 & 99 & 1312 \\
lightning & 171 & 68 & 199 & 500 & 560 & 94 & 1592 &
moon & 126 & 195 & 324 & 500 & 568 & 155 & 1868 &
mountain & 67 & 57 & 319 & 500 & 791 & 195 & 1929 \\
ocean & 54 & 47 & 475 & 500 & 591 & 77 & 1744 &
rain & 71 & 78 & 352 & 500 & 274 & 235 & 1510 &
rainbow & 61 & 44 & 84 & 500 & 231 & 46 & 966 \\
river & 134 & 155 & 558 & 500 & 651 & 111 & 2109 &
snowflake & 175 & 41 & 405 & 500 & 66 & 460 & 1647 &
star & 111 & 204 & 98 & 500 & 61 & 205 & 1179 \\
sun & 248 & 352 & 572 & 500 & 161 & 258 & 2091 &
tornado & 169 & 329 & 373 & 500 & 497 & 211 & 2079 &
 & & & & & & & \\
 
 \cline{1-24}
              \multicolumn{24}{|c|}{\BV{Cold Blooded}}\\
\cline{1-24}
crab & 108 & 50 & 153 & 500 & 717 & 152 & 1680 &
crocodile & 164 & 56 & 120 & 500 & 713 & 161 & 1714 &
fish & 130 & 195 & 429 & 500 & 479 & 373 & 2106 \\
frog & 163 & 118 & 167 & 500 & 761 & 203 & 1912 &
lobster & 243 & 47 & 254 & 500 & 649 & 174 & 1867 &
octopus & 190 & 49 & 331 & 500 & 726 & 149 & 1945 \\
scorpion & 171 & 53 & 133 & 500 & 447 & 455 & 1759 &
sea turtle & 236 & 190 & 410 & 500 & 621 & 254 & 2211 &
shark & 203 & 279 & 269 & 500 & 183 & 532 & 1966 \\
snail & 166 & 18 & 321 & 500 & 465 & 405 & 1875 &
snake & 168 & 57 & 425 & 500 & 501 & 470 & 2121 &
spider & 161 & 154 & 308 & 500 & 593 & 645 & 2361 \\
 
 \cline{1-24}
              \multicolumn{24}{|c|}{\BV{Music}}\\
\cline{1-24}
cello & 93 & 34 & 158 & 500 & 450 & 64 & 1299 &
clarinet & 214 & 25 & 89 & 500 & 407 & 33 & 1268 &
drums & 194 & 18 & 205 & 500 & 769 & 214 & 1900 \\
guitar & 103 & 204 & 203 & 500 & 632 & 183 & 1825 &
harp & 258 & 37 & 224 & 500 & 649 & 45 & 1713 &
piano & 20 & 66 & 296 & 500 & 570 & 119 & 1571 \\
saxophone & 236 & 74 & 358 & 500 & 482 & 310 & 1960 &
trombone & 227 & 195 & 175 & 500 & 484 & 191 & 1772 &
trumpet & 117 & 247 & 122 & 500 & 391 & 188 & 1565 \\
violin & 174 & 282 & 203 & 500 & 512 & 203 & 1874 &
 & & & & & & & &
  & & & & & & & \\
 \hline
\end{tabular}
}

\caption{Detailed statistics of the \DATASETNAME dataset.}
\label{dataset_2}
\end{table*}

\clearpage
\begin{table*}
\centering
\renewcommand{\arraystretch}{1.2}
\renewcommand{\tabcolsep}{1.2mm}
\resizebox{\linewidth}{!}{
\begin{tabular}{|l l l l l l l l |l l l l l l l l|l l l l l l l l|}
\hline
\cline{1-24}
              \multicolumn{24}{|c|}{\BV{Fruit}}\\
\cline{1-24}
\hline
\textbf{class} & \textbf{clp} & \textbf{inf} & \textbf{pnt} & \textbf{qdr} & \textbf{rel} & \textbf{skt} & \textbf{total} &
\textbf{class} & \textbf{clp} & \textbf{inf} & \textbf{pnt} & \textbf{qdr} & \textbf{rel} & \textbf{skt} & \textbf{total} &
\textbf{class} & \textbf{clp} & \textbf{inf} & \textbf{pnt} & \textbf{qdr} & \textbf{rel} & \textbf{skt} & \textbf{total}\\
\hline

apple & 88 & 75 & 445 & 500 & 54 & 181 & 1343 & 
banana & 50 & 376 & 359 & 500 & 258 & 204 & 1747 & 
blackberry & 106 & 214 & 14 & 500 & 568 & 60 & 1462  \\
blueberry & 171 & 110 & 167 & 500 & 733 & 129 & 1810 & 
grapes & 93 & 171 & 318 & 500 & 734 & 287 & 2103 & 
pear & 74 & 115 & 448 & 500 & 438 & 183 & 1758 \\
pineapple & 83 & 139 & 333 & 500 & 673 & 131 & 1859 & 
strawberry & 357 & 308 & 530 & 500 & 454 & 198 & 2347 & 
watermelon & 193 & 401 & 410 & 500 & 671 & 128 & 2303 \\
  \cline{1-24}
              \multicolumn{24}{|c|}{\BV{Sport}}\\
\cline{1-24}
  baseball & 116 & 369 & 122 & 500 & 87 & 48 & 1242  &
baseball bat & 106 & 353 & 145 & 500 & 118 & 196 & 1418 &
basketball & 61 & 219 & 276 & 500 & 237 & 160 & 1453 \\
flying saucer & 233 & 40 & 242 & 500 & 396 & 137 & 1548 &
hockey puck & 188 & 59 & 236 & 500 & 400 & 95 & 1478 &
hockey stick & 155 & 197 & 194 & 500 & 394 & 119 & 1559 \\
snorkel & 278 & 81 & 179 & 500 & 689 & 397 & 2124 &
soccer ball & 85 & 163 & 268 & 500 & 272 & 377 & 1665 &
tennis racquet & 187 & 195 & 18 & 500 & 456 & 202 & 1558\\
yoga & 165 & 447 & 161 & 500 & 371 & 251 & 1895 &
  & & & & & & & &
    & & & & & & & \\
  \cline{1-24}
              \multicolumn{24}{|c|}{\BV{Tree}}\\
\cline{1-24}
bush & 46 & 12 & 67 & 500 & 31 & 626 & 1282 &
cactus & 119 & 36 & 122 & 500 & 658 & 61 & 1496 &
flower & 253 & 140 & 485 & 500 & 360 & 336 & 2074 \\
grass & 148 & 312 & 332 & 500 & 378 & 173 & 1843 &
house plant & 25 & 292 & 416 & 500 & 484 & 156 & 1873 &
leaf & 12 & 96 & 504 & 500 & 414 & 378 & 1904 \\
palm tree & 65 & 277 & 607 & 500 & 333 & 166 & 1948 &
tree & 126 & 511 & 571 & 500 & 536 & 555 & 2799 &
    & & & & & & & \\
  \cline{1-24}
              \multicolumn{24}{|c|}{\BV{Bird}}\\
\cline{1-24}
bird & 336 & 208 & 222 & 500 & 803 & 306 & 2375 &
duck & 142 & 51 & 419 & 500 & 404 & 276 & 1792 &
flamingo & 274 & 39 & 224 & 500 & 542 & 142 & 1721 \\
owl & 133 & 114 & 496 & 500 & 757 & 202 & 2202 &
parrot & 75 & 62 & 336 & 500 & 781 & 266 & 2020 &
penguin & 121 & 201 & 447 & 500 & 700 & 209 & 2178 \\
swan & 469 & 52 & 507 & 500 & 326 & 236 & 2090 &
  & & & & & & & &  
  & & & & & & & \\
  \cline{1-24}
              \multicolumn{24}{|c|}{\BV{Vegetable}}\\
\cline{1-24}
asparagus & 49 & 134 & 408 & 500 & 659 & 209 & 1959 &
broccoli & 105 & 229 & 100 & 500 & 679 & 181 & 1794 &
carrot & 52 & 251 & 265 & 500 & 565 & 34 & 1667 \\
mushroom & 136 & 298 & 254 & 500 & 788 & 252 & 2228 &
onion & 87 & 62 & 471 & 500 & 599 & 158 & 1877 &
peas & 90 & 120 & 90 & 500 & 680 & 81 & 1561 \\
potato & 86 & 61 & 58 & 500 & 608 & 83 & 1396 &
string bean & 139 & 87 & 70 & 500 & 491 & 68 & 1355 &
  & & & & & & & \\
  \cline{1-24}
              \multicolumn{24}{|c|}{\BV{Shape}}\\
\cline{1-24}

circle & 199 & 248 & 292 & 500 & 259 & 202 & 1700 &
hexagon & 196 & 362 & 160 & 500 & 592 & 116 & 1926 &
line & 13 & 210 & 502 & 500 & 102 & 25 & 1352 \\
octagon & 29 & 115 & 19 & 500 & 465 & 117 & 1245 &
squiggle & 148 & 115 & 674 & 500 & 71 & 442 & 1950 &
triangle & 183 & 364 & 298 & 500 & 376 & 303 & 2024 \\
zigzag & 323 & 412 & 110 & 500 & 515 & 144 & 2004 &
  & & & & & & & & 
  & & & & & & & \\
  \cline{1-24}
              \multicolumn{24}{|c|}{\BV{Kitchen}}\\
\cline{1-24}
fork & 200 & 63 & 84 & 500 & 351 & 176 & 1374 &
frying pan & 187 & 68 & 169 & 500 & 399 & 132 & 1455 &
hourglass & 100 & 100 & 206 & 500 & 289 & 134 & 1329 \\
knife & 32 & 108 & 30 & 500 & 582 & 129 & 1381 &
lighter & 66 & 27 & 27 & 500 & 587 & 118 & 1325 &
matches & 60 & 19 & 56 & 500 & 333 & 56 & 1024 \\
spoon & 228 & 127 & 158 & 500 & 534 & 406 & 1953 &

wine bottle & 230 & 442 & 59 & 500 & 407 & 274 & 1912 &
  & & & & & & & \\

  \cline{1-24}
              \multicolumn{24}{|c|}{\BV{Water Transportation}}\\
\cline{1-24}
aircraft carrier & 27 & 88 & 133 & 500 & 390 & 63 & 1201 &
canoe & 68 & 71 & 395 & 500 & 703 & 129 & 1866 &
cruise ship & 208 & 94 & 223 & 500 & 632 & 158 & 1815 \\
sailboat & 162 & 119 & 322 & 500 & 422 & 361 & 1886 &
speedboat & 271 & 76 & 141 & 500 & 620 & 487 & 2095 &
submarine & 344 & 183 & 550 & 500 & 607 & 207 & 2391 \\
  \cline{1-24}
              \multicolumn{24}{|c|}{\BV{Sky Transportation}}\\
\cline{1-24}
airplane & 73 & 62 & 212 & 500 & 218 & 331 & 1396 &
helicopter & 145 & 216 & 257 & 500 & 804 & 200 & 2122 &
hot air balloon & 198 & 48 & 453 & 500 & 732 & 170 & 2101 \\
parachute & 82 & 60 & 140 & 500 & 629 & 233 & 1644 &
  & & & & & & & &
    & & & & & & & \\
  \cline{1-24}
              \multicolumn{24}{|c|}{\BV{Insect}}\\
\cline{1-24}
ant & 81 & 62 & 235 & 500 & 381 & 111 & 1370 &
bee & 202 & 233 & 313 & 500 & 452 & 144 & 1844 &
butterfly & 160 & 162 & 387 & 500 & 658 & 249 & 2116 \\
mosquito & 56 & 232 & 65 & 500 & 562 & 144 & 1559 &
  & & & & & & & &  & & & & & & & \\

  \cline{1-24}
              \multicolumn{24}{|c|}{\BV{Others}}\\
\cline{1-24}
The Mona Lisa & 150 & 112 & 191 & 500 & 289 & 145 & 1387 &
angel & 165 & 17 & 504 & 500 & 31 & 299 & 1516 &
animal migration & 235 & 68 & 604 & 500 & 444 & 112 & 1963 \\
campfire & 122 & 53 & 217 & 500 & 489 & 86 & 1467 &
cannon & 103 & 14 & 54 & 500 & 300 & 93 & 1064 &
dragon & 105 & 30 & 231 & 500 & 485 & 196 & 1547 \\
feather & 268 & 432 & 344 & 500 & 505 & 336 & 2385 &
fire hydrant & 149 & 59 & 29 & 500 & 579 & 148 & 1464 &
mermaid & 207 & 30 & 99 & 500 & 449 & 228 & 1513 \\
snowman & 174 & 123 & 901 & 500 & 114 & 712 & 2524 &
stop sign & 169 & 54 & 87 & 500 & 168 & 109 & 1087 &
teddy-bear & 124 & 407 & 301 & 500 & 528 & 238 & 2098 \\
traffic light & 211 & 280 & 60 & 500 & 379 & 127 & 1557 &
  & & & & & & & & 
  & & & & & & & \\
  \hline

\end{tabular}
}
\caption{Detailed statistics of the \DATASETNAME dataset.}
\label{dataset_3}
\end{table*}